\documentclass{article}
\usepackage[final]{colm2026_conference}

\usepackage{microtype}
\usepackage{hyperref}
\usepackage{url}
\usepackage{booktabs}
\usepackage{amsfonts}
\usepackage{amssymb}
\usepackage{amsmath}
\usepackage{nicefrac}
\usepackage{xcolor}
\usepackage{graphicx}
\usepackage{amsthm}

\newtheorem{proposition}{Proposition}
\newtheorem{remark}{Remark}
\newtheorem{theorem}{Theorem}
\newtheorem{corollary}[theorem]{Corollary}

\usepackage{lineno}

\definecolor{darkblue}{rgb}{0, 0, 0.5}
\hypersetup{colorlinks=true, citecolor=darkblue, linkcolor=darkblue, urlcolor=darkblue}

\title{Super Weights in LLMs and the Failure of Selective Training}

\author{%
  Shreyas Subramanian, \; Adewale Akinfaderin, \; Akarsha Sehwag \\
  Amazon Web Services \\
  \texttt{\{subshrey, akinfaa, akshseh\}@amazon.com} \\
}

\begin{document}

\ifcolmsubmission
\linenumbers
\fi

\maketitle
\raggedbottom

\begin{abstract}
Recent work identified Super Weights, individual parameters whose removal degrades model performance by orders of magnitude. We show that this degradation due to pruning Super Weights does not universally apply to all LLMs. Furthermore, if these parameters are so important, Super Weight-aware training should be effective. We show the opposite. Training Super Weights in isolation (100 to 8,192 parameters) drops accuracy to random-guessing levels on both OLMo-1B and OLMo-7B, and expanding to local neighborhoods of up to 36K parameters provides no improvement. The failure is specific to Super Weight coordinates: training an equal number of randomly chosen positions in the same \texttt{down\_proj} layers instead improves over the baseline, so the collapse comes from targeting Super Weights, not from sparsity itself. Vanilla LoRA, updating every position in attention weight matrices through low-rank structure, succeeds with only 0.16\% of parameters, and applying the same low-rank update to \texttt{down\_proj} succeeds as well. A 10-seed ablation confirms that constraining LoRA updates at positions corresponding to Super Weight coordinates yields statistically indistinguishable results. These findings establish that parameter importance does not imply parameter trainability in isolation, and that effective fine-tuning relies on structured decompositions over entire layers rather than targeting individually important weights.
\end{abstract}

\section{Introduction}

Parameter-efficient fine-tuning (PEFT) methods like LoRA \citep{hu2022lora} achieve performance comparable to full fine-tuning while updating only 0.1--1\% of parameters. \citet{aghajanyan2021intrinsic} showed that fine-tuning operates in a low intrinsic dimension, suggesting that even random low-dimensional subspaces suffice. A natural hypothesis follows: if random subspaces work, then deliberately targeting the most important parameters should work at least as well. Recent work by \citet{yu2025SuperWeight} identified Super Weights, individual parameters whose removal increases perplexity by orders of magnitude, providing an ideal test case for this hypothesis. We show that the hypothesis is wrong. Training Super Weights and their neighborhoods in isolation fails completely, and this failure is specific to those coordinates rather than to sparse training in general, establishing that parameter importance and parameter trainability in isolation are fundamentally different properties.

We investigate these questions through experiments on OLMo-1B and OLMo-7B, spanning pruning, direct training, neighborhood training, LoRA-dproj-SW-freeze, LoRA-$\Delta W$-SW-freeze, and multi-seed ablation studies. We begin by addressing a methodological limitation of prior work: \citet{yu2025SuperWeight} identified Super Weights from a single forward pass. We validate consistency across 20 diverse samples, finding that 9 weight positions create activation spikes in 100\% of inputs, confirming Super Weights are structural properties of the pretrained model rather than input-specific artifacts. 

Our experiments then reveal a clear pattern. Training only Super Weights (100--8,192 parameters) drops accuracy to random-guessing levels, and expanding to local neighborhoods of up to 36K parameters provides no improvement. Two controls on the \texttt{down\_proj} layers isolate the cause. Training an equal number of randomly chosen \texttt{down\_proj} positions (Super Weight coordinates excluded) improves over the baseline rather than collapsing, so sparsity alone is not the problem. Applying a low-rank update to \texttt{down\_proj} also succeeds, so the module is not the problem either. The collapse is specific to targeting Super Weight coordinates and their entangled neighborhoods. Meanwhile, vanilla LoRA achieves 67\% accuracy (above the 61\% baseline) by updating every position in the weight matrix through rank-8 decomposition, despite training only 0.16\% of parameters. LoRA-dproj-SW-freeze (freezing attention positions at coordinates matching \texttt{down\_proj} Super Weight indices) and LoRA-$\Delta W$-SW-freeze (freezing top magnitude weights in the pre-trained $\Delta W$) have no measurable effect. A 10-seed ablation shows statistically indistinguishable results ($p > 0.05$), with 80\% of seeds producing identical predictions. LoRA succeeds because its low-rank structure coordinates updates across entire layers, while training isolated Super Weight positions breaks the relational structure those weights depend on.

Our main contributions are: (1) validating Super Weight consistency across diverse models, extending the single forward pass analysis of \citet{yu2025SuperWeight}; (2) showing that training Super Weights and their neighborhoods in isolation fails regardless of parameter count, while an equal-size random-position control on the same \texttt{down\_proj} layers succeeds, isolating the failure to the Super Weight coordinates rather than to sparsity or module choice; (3) demonstrating that constraining attention LoRA updates at positions corresponding to Super Weight coordinates has no measurable effect, showing that the attention representations feeding into Super Weight layers are robust to position-level restrictions; (4) showing that LoRA's own highest-magnitude update positions are not individually critical, in contrast to base model Super Weights; (5) identifying layer-wide coordination as the key factor separating successful and unsuccessful fine-tuning strategies; and (6) providing theoretical analysis from two complementary perspectives: an intrinsic dimensionality argument showing that sparse updates capture only a vanishing fraction of the fine-tuning subspace, and an optimizer dynamics argument showing that curvature at Super Weight positions implicitly suppresses their updates (Appendices~\ref{app:theory_intrinsic}--\ref{app:theory_optimizer}).

\section{Related work}

Recent work has revealed that individual parameters and activations in LLMs carry disproportionate importance. \citet{dettmers2022gpt3} first documented emergent outlier features in transformer hidden states at scale, showing that a small set of activation dimensions contain values far larger than others and that these outliers are critical for preserving model quality during quantization. \citet{sunmassive} further characterized these as ``massive activations": persistent, position-fixed activation outliers that remain constant regardless of input. \citet{yu2025SuperWeight} traced these activation outliers to their source, identifying Super Weights: individual scalar parameters whose removal increases perplexity by orders of magnitude. At a broader level, \citet{liu2025relation} demonstrated that groups of neurons in LLMs encode relation-specific knowledge, exhibiting cumulativity (distributed across many neurons), versatility (shared across relations), and interference (deactivating one relation's neurons can affect others). Our work extends this line of research by validating Super Weight consistency across inputs and investigating whether these critical parameters can be successfully fine-tuned in isolation.

Parameter-efficient fine-tuning (PEFT) methods aim to adapt large pretrained models while updating only a small fraction of parameters. LoRA \citep{hu2022lora} achieves this through low-rank decomposition of weight updates, training only small matrices $A$ and $B$ such that $\Delta W = BA$. \citet{zhangadaptive} extended this with adaptive rank allocation across layers based on importance scores. At the opposite extreme, \citet{zaken2022bitfit} showed that updating only bias terms can achieve competitive performance on certain tasks, demonstrating that selective parameter training can sometimes succeed. \citet{aghajanyan2021intrinsic} provided a theoretical foundation by showing that fine-tuning operates in a very low intrinsic dimension: as few as 200 randomly projected parameters can capture 90\% of full fine-tuning performance for RoBERTa on MRPC, and pretraining implicitly minimizes this intrinsic dimension. Our work complements these findings by showing that while low intrinsic dimensionality explains \textit{why} few parameters suffice, the \textit{structure} of those parameters matters critically: sparse-subset selection fails where low-rank full-layer decompositions succeed.

Understanding which parameters matter in large models has also been studied through pruning and sparsity. \citet{jaiswal2023emergence} identified ``essential sparsity" in pretrained transformers: a sharp threshold beyond which one-shot magnitude pruning causes rapid performance collapse. They showed that 30--50\% of weights can be removed without retraining, and that simple magnitude pruning matches expensive Lottery Ticket methods within this sparsity range. This parallels our observation that LoRA's implicit regularization provides natural protection for important weights without explicit constraints. While these works focus on identifying or removing unimportant parameters, our work asks the converse question: given that we know which parameters are most critical (Super Weights), can we fine-tune \textit{only} those parameters? Our finding is that this always fails, even with neighborhoods of up to 36K parameters, establishing that parameter importance and parameter trainability in isolation are different properties.

\section{Methods}

\subsection{Experimental setup}

We conduct all training experiments (direct SW training, neighborhood training, LoRA variants, seed ablation) on two models: OLMo-1B \citep{groeneveld2024olmo} (1.28B parameters, 16 transformer layers) and OLMo-7B (7B parameters, 32 transformer layers). For the pruning replication, we additionally test Phi-3-mini (3.8B), Meta-Llama-3-8B, Llama-3.1-8B-Instruct, Llama-3.2-3B, Llama-3.2-1B, Qwen2.5-1.5B-Instruct, and Gemma-2-9B-it to assess generality of the Super Weight pruning phenomenon across architectures (10 models total). Our primary evaluation dataset is ARC-Easy \citep{clark2018arc} (2,251 train, 2,376 test), with Winogrande \citep{sakaguchi2019winogrande} (9,248 train, 1,267 test) for additional validation. Pretrained baseline accuracies are reported in Table~\ref{tab:baseline}. We identify Super Weights by magnitude ranking, scanning \emph{all} weight matrices (\texttt{q\_proj}, \texttt{k\_proj}, \texttt{v\_proj}, \texttt{o\_proj}, \texttt{gate\_proj}, \texttt{up\_proj}, \texttt{down\_proj}) across all layers. Empirically, the highest-magnitude parameters concentrate overwhelmingly in \texttt{down\_proj}: 80\% of the top-10 and 52\% of the top-100 highest-magnitude parameters across the entire model are in \texttt{down\_proj}. This is not a design choice but an empirical finding, consistent with \citet{yu2025SuperWeight}, who trace activation spikes directly to \texttt{down\_proj} magnitude outliers. We cache the top 10,000 positions per model.

All training experiments use AdamW with learning rate $10^{-4}$, 3 epochs, and an effective batch size of 16 (batch size 4, gradient accumulation 4 for 1B; batch size 1, gradient accumulation 16 for 7B). For LoRA experiments, we use rank $r=8$, $\alpha=16$, targeting attention layers (q\_proj, k\_proj, v\_proj, o\_proj), yielding 2.1M trainable parameters (0.16\% of OLMo-1B). All methods are evaluated on held-out test sets using exact match accuracy, selecting the answer with lowest perplexity for multiple-choice questions.

\subsection{Experiments}

We design six experiments to test whether Super Weights can be trained in isolation, whether local context around them is sufficient, and whether full-layer methods like LoRA require any special treatment of Super Weight positions.

\begin{enumerate}

\item \textbf{Super Weight consistency:} \citet{yu2025SuperWeight} identified Super Weights from a single forward pass, leaving open the question of whether the same positions are critical across different inputs. We run 1,000 random WikiText-2 \citep{merity2016wikitext} samples through the model, identify weights that produce the largest activation spikes in down\_proj layers for each sample, and measure how consistently the same positions appear. Each sample requires only a single gradient-free forward pass, so the cost is linear in the number of samples and completes in minutes on one GPU.

\item \textbf{Pruning replication:} We replicate the pruning experiment of \citet{yu2025SuperWeight} across six models. For OLMo-1B/7B we use our magnitude-based identification; for Phi-3-mini and Meta-Llama-3-8B we use coordinates from \citet{yu2025SuperWeight}; for Qwen2.5-1.5B and Llama-3.2-1B we test magnitude-based identification to probe generalisation.

\item \textbf{Direct Super Weight training:} We train only the top-$k$ Super Weights by magnitude while freezing all other parameters, for $k \in \{100, 1000, 4096, 8192\}$.

\item \textbf{Neighborhood training:} If isolated Super Weights fail because they lack the local context needed to coordinate with adjacent parameters, then unfreezing their immediate neighbors should help. We test this by extending direct training to include all parameters within radius 1 of each Super Weight, meaning that for a Super Weight at matrix position $(i,j)$ we also train the eight entries at $(i\pm1, j\pm1)$, forming a $3\times3$ patch of roughly nine parameters per Super Weight. For $k=100$ this yields approximately 900 trainable parameters; for $k=1000$, approximately 9,000; and for $k=4096$, approximately 36,864.

\item \textbf{LoRA-dproj-SW-freeze:} LoRA targets attention projections (\texttt{q\_proj}, \texttt{k\_proj}, \texttt{v\_proj}, \texttt{o\_proj}), while Super Weights reside in \texttt{down\_proj} (MLP), which are entirely different weight matrices with no overlap. LoRA never directly modifies Super Weight values; instead it reshapes the representations fed into \texttt{down\_proj} via the residual stream. This raises a natural question: are the attention representations that flow \emph{into} Super Weight-containing layers sensitive to which attention positions get updated? To test this, we implement a variant that applies a scale factor $s$ to attention update positions whose indices correspond to Super Weight coordinates in \texttt{down\_proj} (i.e., position $(i,j)$ in the attention LoRA update is scaled when $(i,j)$ matches a Super Weight location in \texttt{down\_proj}):
\begin{equation}
\Delta W = (B \odot M_s) \times A \times \frac{\alpha}{r}
\end{equation}
where $M_s[i,j] = s$ if $(i,j)$ matches a Super Weight coordinate and 1.0 otherwise. We test $s \in \{0.0, 0.1, 0.2, 0.5, 0.8, 1.0\}$, where $s=0.0$ fully restricts those positions and $s=1.0$ recovers vanilla LoRA. To validate robustness, we run a 10-seed ablation (seeds 42--51) comparing $s=0.0$ and $s=1.0$.

\item \textbf{LoRA-$\Delta W$-SW-freeze:} The previous experiments establish that LoRA succeeds without touching base model Super Weights. A natural follow-up question is: does LoRA itself develop analogous critical positions? We identify ``LoRA-$\Delta W$ Super Weights'' by loading a trained vanilla LoRA checkpoint, computing $\Delta W = B A \cdot \alpha/r$ for each attention layer, and finding the top-1,000 highest-magnitude positions in $\Delta W$ across all layers. We then further train the LoRA from with those positions frozen (gradient hooks on the corresponding rows of \texttt{lora\_B} and columns of \texttt{lora\_A}), and compare to vanilla LoRA across 3 seeds on both OLMo-1B and OLMo-7B.

\end{enumerate}

\subsection{Weight change analysis}
For each trained model, we measure how much Super Weight positions changed during training. For LoRA, we reconstruct the full update as $\Delta W = BA \cdot \alpha/r$ and extract values at Super Weight positions. For direct training, we compare checkpoints before and after training. We report absolute change, relative change (\%), and coefficient of variation across Super Weight positions.

\section{Results}

We present results in two stages. First, we validate the Super Weight phenomenon across nine diverse models spanning multiple architectures and scales (Section~\ref{sec:pruning}). This broad validation uses our magnitude-based identification for newer models, and the original activation-spike coordinates from \citet{yu2025SuperWeight} for models they tested. Second, we conduct a comprehensive investigation of selective training and LoRA behaviour exclusively on OLMo-1B and OLMo-7B, where Super Weight criticality is confirmed and we have full experimental control. All OLMo results are presented side-by-side for both scales. For ARC-Easy (4-choice questions), random guessing corresponds to approximately 25\% accuracy.

\subsection{Baseline performance}

Pretrained baselines across all 10 models are reported in Table~\ref{tab:baseline} (Appendix~\ref{app:baselines}). OLMo-1B and OLMo-7B are used throughout all experiments; the remaining 8 models are evaluated only in the pruning replication (Section~\ref{sec:pruning}).

\subsection{Super Weight consistency}

\begin{table}[h!]
\centering
\caption{Super Weight consistency across 1,000 random WikiText-2 samples (OLMo-1B), measured by how often the same weight positions produce the largest activation spikes in down\_proj layers. Each ``Count'' row is cumulative (positions appearing in at least the stated fraction of samples).}
\label{tab:consistency}
\small
\begin{tabular}{@{}lrr@{}}
\toprule
\textbf{Consistency Level} & \textbf{Count} & \textbf{Sample Coverage} \\
\midrule
100\% consistent & 9 weights & 1000/1000 samples \\
$\geq$50\% consistent & 13 weights & 500--633/1000 samples \\
$\geq$20\% consistent & 24 weights & 200--498/1000 samples \\
Total unique weights & 58 weights & --- \\
\bottomrule
\end{tabular}
\end{table}

\citet{yu2025SuperWeight} identified Super Weights from a single forward pass. Table~\ref{tab:consistency} shows that the same positions are critical across diverse inputs: 9 weight positions produce the largest activation spikes in all 1,000 samples, and 13 positions appear in at least half. The set stabilizes quickly, reaching its final composition by roughly $n=100$, and there is a sharp gap between the 9 perfectly consistent positions and the next tier at 63.3\%, indicating a well-defined structural core. The 9 perfectly consistent positions span layers 1--12 of the 16-layer model. This confirms that Super Weights are structural properties of the pretrained model rather than artifacts of particular inputs. Having established their consistency, we next test whether they are truly critical by replicating the pruning experiment across multiple model families.

\subsection{Pruning replication}
\label{sec:pruning}

Zeroing the Super Weight coordinates from \citet{yu2025SuperWeight} collapses WikiText-2 perplexity by orders of magnitude in all five models tested with paper coordinates (Table~\ref{tab:pruning}, Appendix~\ref{app:pruning}). For OLMo-1B, perplexity increases 3,663$\times$ and ARC-Easy accuracy drops to near random-guessing (27.0\%). For OLMo-7B, perplexity collapses even more severely (4,400$\times$, 9.59\,$\to$\,42,024), while ARC-Easy accuracy drops more modestly (73.3\%\,$\to$\,60.5\%); lowest-perplexity selection over four fixed choices can survive substantial generation degradation, so the perplexity metric reflects the collapse more directly than multiple-choice accuracy at 7B scale. We confirm the same perplexity collapse in Phi-3-mini, Mistral-7B, and Meta-Llama-3-8B. Applying magnitude-based identification to five additional models produced no measurable effect, revealing that high magnitude alone is insufficient. The phenomenon depends on activation outlier patterns specific to each model family. Full results and discussion are in Appendix~\ref{app:pruning}.

All subsequent experiments focus exclusively on OLMo-1B and OLMo-7B, where Super Weight criticality is confirmed.

\subsection{Direct Super Weight training}

Training only Super Weights drops accuracy to random-guessing levels on both models regardless of how many positions are selected (Table~\ref{tab:direct_training}). Increasing from 100 to 8,192 trained positions ($81\times$ more parameters) provides no improvement. Training loss decreases during optimization, but validation perplexity explodes, indicating that isolated positions memorize training data without generalizing. We next ask whether adding local context around each Super Weight resolves this failure.

\begin{table}[h!]
\centering
\caption{ARC-Easy accuracy when training only the top-$k$ Super Weights by magnitude, all other parameters frozen. Baselines: OLMo-1B 60.65\%, OLMo-7B 73.3\%.}
\label{tab:direct_training}
\small
\begin{tabular}{@{}lccc@{}}
\toprule
\textbf{Method} & \textbf{Params} & \textbf{OLMo-1B} & \textbf{OLMo-7B} \\
\midrule
Top 100 SW  & 100   & 25.30\% & 26.3\% \\
Top 1,000 SW & 1,000 & ${\sim}$25\% & 26.1\% \\
Top 4,096 SW & 4,096 & 26.30\% & 26.1\% \\
Top 8,192 SW & 8,192 & ${\sim}$25\% & ${\sim}$25\% \\
\bottomrule
\end{tabular}
\end{table}

\begin{figure}[t]
\centering
\includegraphics[width=\textwidth]{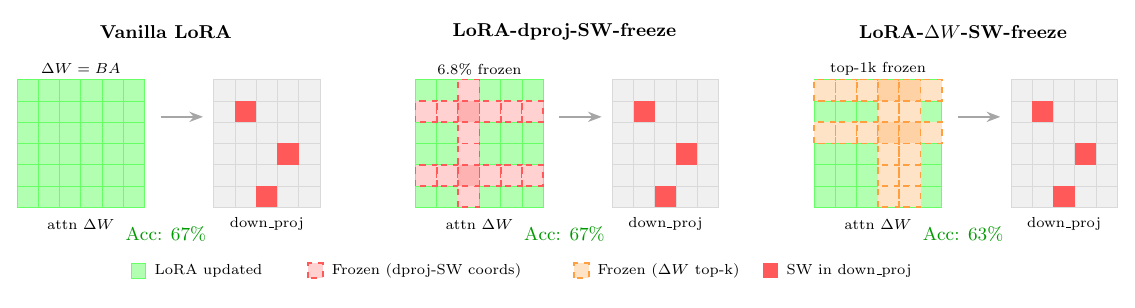}
\caption{The three LoRA variants. All operate on attention projections (\texttt{q/k/v/o\_proj}) and leave \texttt{down\_proj} (where Super Weights reside) completely unchanged. \textbf{Vanilla LoRA:} all attention positions updated via $\Delta W = BA$. \textbf{LoRA-dproj-SW-freeze:} rows and columns at indices corresponding to \texttt{down\_proj} Super Weight coordinates are frozen (6.8\% of LoRA params for OLMo-1B). \textbf{LoRA-$\Delta W$-SW-freeze:} rows and columns at the top-1,000 highest-magnitude positions in the trained $\Delta W$ are frozen. All three variants achieve $\geq$63\% accuracy, which is statistically indistinguishable.}
\label{fig:lora_variants}
\end{figure}

\subsection{Neighborhood training}

\begin{table}[t]
\centering
\caption{ARC-Easy accuracy when training Super Weights plus all parameters within radius 1 in the weight matrix.}
\label{tab:neighborhood}
\small
\begin{tabular}{@{}lccc@{}}
\toprule
\textbf{Method} & \textbf{Params} & \textbf{OLMo-1B} & \textbf{OLMo-7B} \\
\midrule
Neighborhood ($k$=100)   & ${\sim}$900    & 24.96\% & 26.3\% \\
Neighborhood ($k$=1,000) & ${\sim}$9,000  & 25.59\% & 26.1\% \\
Neighborhood ($k$=4,096) & ${\sim}$36,864 & 25.59\% & 26.1\% \\
\bottomrule
\end{tabular}
\end{table}

Including local neighborhoods provides no benefit on either model (Table~\ref{tab:neighborhood}). Even with 36,864 trainable parameters, accuracy remains at random-guessing levels. Training becomes unstable on OLMo-1B, where loss increases from 2.00 to 201.92 for the $k$=4,096 variant. Having shown that training at Super Weight coordinates and their neighborhoods fails, we turn to full-layer LoRA methods and examine whether Super Weights require any special treatment within them.

\begin{figure}[t]
\centering
\includegraphics[width=0.95\textwidth]{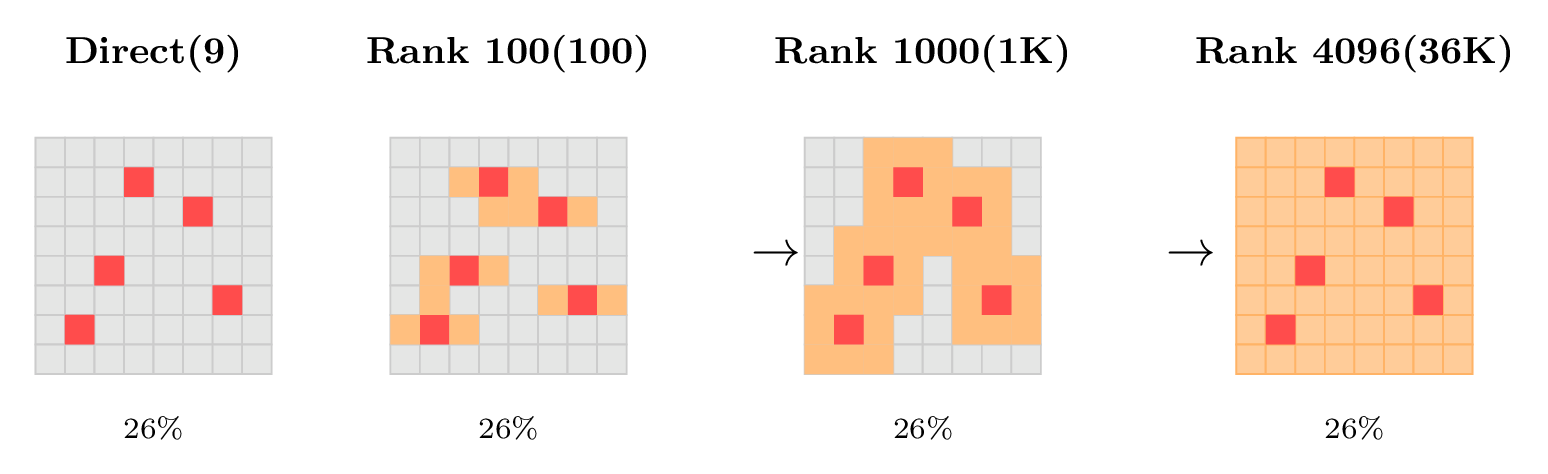}
\caption{Expanding from isolated Super Weights to neighborhoods of 100, 1,000, or 36,864 parameters (orange) does not improve over random guessing. Here, rank corresponds to neighborhood size.}
\label{fig:neighborhood}
\end{figure}

\subsection{LoRA-dproj-SW-freeze}

Majority of Super Weights happen to reside in \texttt{down\_proj} while LoRA targets attention projections, which are entirely different weight matrices. LoRA reshapes attention outputs that feed directly into \texttt{down\_proj}. We scale down LoRA's update at attention positions whose indices match Super Weight coordinates (see Appendix~\ref{app:lora_geometry} for the full geometric explanation). At $s=0.0$, this affects 6.8\% of LoRA parameters for OLMo-1B and 0.5\% for OLMo-7B.

\begin{table}[t]
\centering
\caption{LoRA-dproj-SW-freeze: scale factors $s$ applied to attention update positions whose indices match \texttt{down\_proj} Super Weight coordinates. $s$=0.0 (LoRA-dproj-SW-freeze) fully restricts those positions; $s$=1.0 is vanilla LoRA.}
\label{tab:superlora}
\small
\begin{tabular}{@{}llrr@{}}
\toprule
\textbf{Scale} & \textbf{Description} & \textbf{OLMo-1B} & \textbf{OLMo-7B} \\
\midrule
1.0 & Vanilla LoRA          & 66.88\% & 77.3\% \\
0.8 & 80\% of matched positions & 66.54\% & 76.9\% \\
0.5 & 50\% of matched positions & 66.50\% & 78.0\% \\
0.2 & 20\% of matched positions & 66.46\% & 76.6\% \\
0.1 & 10\% of matched positions & 66.50\% & 76.4\% \\
0.0 & LoRA-dproj-SW-freeze      & 66.50\% & 76.7\% \\
\bottomrule
\end{tabular}
\end{table}

All LoRA variants improve well above baseline on both models regardless of how aggressively the matched attention positions are restricted (Table~\ref{tab:superlora}). On OLMo-1B, all variants achieve 66.46--66.88\%; on OLMo-7B, 76.4--78.0\%. The spread is less than 1 percentage point in both cases. Restricting even 100\% of the matched positions ($s=0.0$), which freezes 6.8\% of LoRA parameters for OLMo-1B and 0.5\% for OLMo-7B, does not degrade performance, indicating that LoRA's low-rank structure compensates via the remaining 93--99.5\% of unrestricted parameters. We next ask whether LoRA's own highest-magnitude update positions are similarly non-critical.

\subsection{LoRA-$\Delta W$-SW-freeze}

\begin{table}[t]
\centering
\caption{LoRA-$\Delta W$-SW-freeze: accuracy when the top-1,000 highest-magnitude positions in the trained LoRA update $\Delta W$ are frozen in a fresh training run. 3 seeds each.}
\label{tab:lora_sw_freeze}
\small
\begin{tabular}{@{}llrr@{}}
\toprule
\textbf{Method} & \textbf{Frozen} & \textbf{OLMo-1B} & \textbf{OLMo-7B} \\
\midrule
Vanilla LoRA               & none     & 62.72\% $\pm$ 0.62\% & 65.07\% $\pm$ 1.46\% \\
LoRA-$\Delta W$-SW-freeze & 1,000 positions & 62.77\% $\pm$ 0.62\% & 64.09\% $\pm$ 3.38\% \\
\bottomrule
\end{tabular}
\end{table}

We identify the top-1,000 highest-magnitude positions in the trained LoRA update $\Delta W$ and ask whether freezing them during a fresh training run degrades performance. For OLMo-1B, vanilla LoRA achieves 62.72\% $\pm$ 0.62\% across 3 seeds; LoRA-$\Delta W$-SW-freeze achieves 62.77\% $\pm$ 0.62\%, which is statistically identical. Base model Super Weights cause severe degradation when removed (perplexity increasing by orders of magnitude), but LoRA's highest-magnitude update positions are not individually critical. LoRA's solution is distributed across its parameter space with no single bottleneck positions.

\subsection{Summary of results}

\begin{table}[t]
\centering
\caption{ARC-Easy accuracy across all methods on both models. Selective methods train sparse subsets of coordinates; full-layer methods update every position through low-rank structure. The two \texttt{down\_proj} controls (OLMo-1B) hold the module fixed and vary only the update structure. $^\dagger$~7B pruning accuracy drop is less dramatic than 1B due to larger model capacity, but perplexity collapsed 4,400$\times$ (baseline 9.59~$\to$~42,024), confirming Super Weight criticality (Table~\ref{tab:pruning}). $^*$~Mean of 3 seeds; see Table~\ref{tab:lora_sw_freeze}. $^\ddagger$~Single seed (seed 42), OLMo-1B only.}
\label{tab:comprehensive}
\small
\begin{tabular}{@{}lrrrr@{}}
\toprule
\textbf{Method} & \textbf{Params} & \textbf{Update Type} & \textbf{OLMo-1B} & \textbf{OLMo-7B} \\
\midrule
\multicolumn{5}{l}{\textit{Training at Super Weight coordinates}} \\
Pruning (2 SW)              & 0      & Remove       & 26.98\% & 60.5\%$^\dagger$ \\
Direct ($k$=100)            & 100    & Sparse       & 25.30\% & 26.3\% \\
Direct ($k$=4,096)          & 4,096  & Sparse       & 26.30\% & 26.1\% \\
Neighborhood ($k$=100)      & 900    & Sparse       & 24.96\% & 26.3\% \\
Neighborhood ($k$=4,096)    & 36,864 & Sparse       & 25.59\% & 26.1\% \\
\midrule
\multicolumn{5}{l}{\textit{Controls on \texttt{down\_proj} (same module)}} \\
Random sparse ($k$=4,096)   & 4,096  & Sparse       & 64.18\%$^\ddagger$ & --- \\
LoRA on \texttt{down\_proj} & 1.31M  & Full layer   & 68.77\%$^\ddagger$ & --- \\
\midrule
\multicolumn{5}{l}{\textit{Full-Layer LoRA}} \\
Vanilla LoRA              & 2.1M & Full layer & 66.88\% & 77.3\% \\
LoRA-dproj-SW-freeze      & 2.1M & Full layer & 66.50\% & 76.7\% \\
LoRA-$\Delta W$-SW-freeze (top-1k) & 2.1M & Full layer & 62.77\%$^*$ & 64.09\% \\
\bottomrule
\end{tabular}
\end{table}

Every method that trains at Super Weight coordinates falls to random-guessing levels regardless of parameter count (Table~\ref{tab:comprehensive}). The two \texttt{down\_proj} controls localize the cause. With the module held fixed, training an equal number of random coordinates ($k$=4,096) reaches 64.18\%, and a rank-8 low-rank update reaches 68.77\%. Both improve over the 60.65\% baseline, so neither sparsity nor the choice of \texttt{down\_proj} explains the collapse. What fails is training directed at the Super Weight coordinates themselves. All full-layer LoRA variants improve well above baseline, and the pattern is consistent across the 1B and 7B models. (LoRA-$\Delta W$-SW-freeze uses 3 seeds; see Table~\ref{tab:lora_sw_freeze}.)

\subsection{Super Weight stability and weight change analysis}

Super Weights remain stable under LoRA training: all maintain their top-1\% percentile rank, all remain local maxima, and their magnitude changes by only $-$0.03\% on average (Table~\ref{tab:sw_stability}, Appendix~\ref{app:stability}). Direct training produces changes 55$\times$ larger than LoRA (8.19\% vs 0.15\% average), yet LoRA succeeds while direct training does not (Table~\ref{tab:weight_changes}, Appendix~\ref{app:stability}). The size of the update alone does not determine outcome; coordinated full-layer structure does.

\subsection{Seed ablation}

A 10-seed ablation (seeds 42--51) confirms that vanilla LoRA and LoRA-dproj-SW-freeze ($s$=0.0) are statistically indistinguishable: identical mean accuracy (62.90\% $\pm$ 0.45\%) on OLMo-1B with 8/10 seeds producing bit-for-bit identical predictions ($p > 0.05$), and all 10 seeds identical on OLMo-7B ($p = 1.0$). Full results are in Table~\ref{tab:seed_ablation} and Figure~\ref{fig:seed_ablation} (Appendix~\ref{app:seed_ablation}).

\section{Discussion}

Table~\ref{tab:comprehensive} summarizes our main results. Training Super Weights and their neighborhoods in isolation fails at random-guessing levels, while all full-layer LoRA variants succeed equivalently. Two controls localize the cause, and we discuss the findings below across two main observations.

\paragraph{Observation 1: The failure is specific to Super Weight coordinates, not to sparsity or module.} Training Super Weights and their neighborhoods in isolation fails at random-guessing levels regardless of parameter count. Two controls on the \texttt{down\_proj} layers show that neither sparsity nor the choice of module explains this. Training an equal number of randomly chosen \texttt{down\_proj} positions, with Super Weight coordinates excluded, reaches 64.18\% on OLMo-1B ARC-Easy, above the 60.65\% baseline, so sparse training in the same layers succeeds when it avoids Super Weight coordinates. Applying a rank-8 low-rank update to \texttt{down\_proj} reaches 68.77\%, so the module supports effective adaptation. What collapses is training directed at the Super Weight coordinates themselves. When only those positions receive gradients, each trained position changes independently while its surrounding row and column stay frozen, which breaks the relational structure the pretrained model learned. Under direct training, 35.5\% of Super Weight updates exceed 10\% of their original value, with some changing by as much as 30.89\%. These large, uncoordinated changes at isolated outlier positions prevent generalization, an effect the random-position control does not exhibit.

Recent work by \citet{morris2026tinylora} shows that full-layer low-rank updates can succeed with as few as 13 trainable parameters. The distinction from our sparse training is that TinyLoRA's frozen SVD components maintain layer-wide coordination, while our sparse training has no coordination mechanism (see Appendix~\ref{app:obs4} for the full comparison).

Expanding to local neighborhoods does not resolve this. A fixed neighborhood around each Super Weight captures at most 36,864 parameters, but the computations that Super Weights participate in span entire layers and propagate across multiple layers through residual connections. Theoretical analysis (Appendix~\ref{app:theory_intrinsic}) shows that sparse coordinate-aligned updates capture only a $k/N$ fraction of the fine-tuning subspace, while LoRA's rank-$r$ manifold provides $256\times$ more capacity.

\paragraph{Observation 2: LoRA produces coordinated updates across entire layers.} LoRA succeeds because its low-rank decomposition $\Delta W = BA \cdot \alpha/r$ produces an update at every position $(i,j)$ in the weight matrix:
\begin{equation}
\Delta W_{ij} = \sum_{k=1}^{r} B_{ik} A_{kj} \times \frac{\alpha}{r}
\end{equation}
While only the rank-$r$ matrices $B$ and $A$ are trainable (2.1M parameters for $r=8$), they generate updates for all 4.2M positions. This contrasts with sparse-subset training, where 36K positions are updated and the remaining 4.16M (99.1\%) stay frozen. Under LoRA, no position is left unchanged.

The low-rank constraint has several useful properties in this context. Updates to any position $(i,j)$ are coupled to all other positions sharing row $i$ or column $j$ through the shared bottleneck, so gradients cannot isolate single positions. The rank constraint also limits updates to an $r$-dimensional subspace (0.39\% of full rank for $r=8$, $n=2048$), producing small, structured changes.

Super Weights predominantly reside in \texttt{down\_proj} (MLP output projection), while LoRA targets attention projections (\texttt{q\_proj}, \texttt{k\_proj}, \texttt{v\_proj}, \texttt{o\_proj}). LoRA therefore produces zero direct change to Super Weight values, but is connected through the residual stream: attention outputs feed into the MLP block, and \texttt{down\_proj} processes those representations. LoRA learns to route representations through the attention layers so that the fixed, unchanged Super Weights in \texttt{down\_proj} produce task-appropriate outputs. Scaling down the attention updates at positions corresponding to Super Weight coordinates, which restricts 24--100\% of the affected positions depending on coordinate bounds, has no measurable effect: the 10-seed ablation shows $p > 0.05$ with 80\% of seeds producing identical predictions.

Additionally, top Super Weights within $\Delta W$ when frozen also achieve performance on par with vanilla LoRA. Super Weights act as fixed high-gain amplifiers, and LoRA adjusts what gets fed into them. Direct training of Super Weights fails because modifying these amplifiers in isolation, without coordinating the representations that flow through them, breaks the structure the pretrained model depends on. Theoretical analysis (Appendix~\ref{app:theory_optimizer}) confirms that adaptive optimizers implicitly suppress updates at Super Weight positions by $1/M$ due to $M^2$-fold larger curvature, making explicit freezing redundant.

LoRA's full-layer low-rank structure is robust to position-level restrictions. The representations that flow into Super Weight layers are shaped by the overall low-rank update, not by any particular set of positions within it. Constraining a subset of attention positions leaves the downstream Super Weights equally well-served.

We see that all LoRA variants improve well above baseline regardless of how Super Weights are treated within the full-layer update. This suggests that the structure of the update (full-layer, low-rank) matters, while position-level treatment within that structure does not. Super Weights must remain at their positions (pruning is destructive), and they must be updated as part of a full layer (isolation fails), but they do not require special handling within full-layer methods.

\subsection{Limitations}

We evaluate primarily on ARC-Easy, with Winogrande for validation (Appendix~\ref{app:winogrande}). On Winogrande no method improves over the strong pretrained baseline, but the relative pattern holds: training at Super Weight coordinates collapses to chance while LoRA stays close to baseline. The failure of training at Super Weight coordinates (dropping to random-guessing levels) is a large effect unlikely to be task-specific, but validating on harder benchmarks (MMLU, GSM8K) and with PEFT methods beyond LoRA (adapters, prompt tuning) would strengthen generality.

Our random-position control uses a single seed (seed 42) at $k=4{,}096$ on OLMo-1B. Repeating it across more seeds, more parameter budgets, and both OLMo scales would further confirm that the collapse is specific to Super Weight coordinates, and we leave this broader sweep to future work.

Our training experiments use the activation-spike Super Weight coordinates of \citet{yu2025SuperWeight}, whose criticality we confirm through pruning replication (Section~\ref{sec:pruning}). We use magnitude ranking only for consistency validation and as a cross-model identification baseline, where we report that it does not generalise to all architectures (Appendix~\ref{app:pruning}). Our consistency validation uses WikiText-2 only; cross-dataset validation would strengthen claims about structural properties. Extending the training experiments to non-OLMo families using activation-spike identification is an important direction.

\section{Conclusion}

Across all experiments, training at Super Weight coordinates and their neighborhoods fails at random-guessing levels regardless of parameter count (100 to 36,864), while LoRA succeeds by updating every position in the weight matrix through low-rank decomposition. The failure is specific to those coordinates rather than to sparse training in general: an equal-size random-position control on the same \texttt{down\_proj} layers improves over the baseline, and a low-rank update to \texttt{down\_proj} succeeds as well. Position-aware constraints within LoRA, as well as direct freezing of LoRA Super Weights, have no measurable effect, confirming that full-layer coordination drives successful adaptation. These results suggest that PEFT methods should derive parameter efficiency from structured decompositions over entire layers. Theoretical analysis (Appendices~\ref{app:theory_intrinsic}--\ref{app:theory_optimizer}) supports these findings. Sparse coordinate-aligned updates capture only a vanishing fraction of the fine-tuning subspace, so they are suboptimal in general, and the additional curvature at Super Weight positions turns that suboptimality into catastrophic collapse when training targets those coordinates directly. We hope this work encourages further investigation into the role of layer-wide update structure in fine-tuning, particularly across larger models, diverse tasks, and alternative PEFT approaches.

\bibliography{colm2026_conference}
\bibliographystyle{colm2026_conference}

\newpage
\appendix

\section{Baseline Performance}
\label{app:baselines}

\begin{table}[h]
\centering
\caption{Pretrained baseline performance (zero-shot). Models above the rule are used for all experiments; models below are tested for pruning replication only. All models evaluated on ARC-Easy and Winogrande; WikiText-2 PPL evaluated in float16 except where noted.}
\label{tab:baseline}
\small
\begin{tabular}{@{}lccc@{}}
\toprule
\textbf{Model} & \textbf{ARC-Easy} & \textbf{Winogrande} & \textbf{WikiText-2 PPL} \\
\midrule
OLMo-1B (1.28B)         & 60.6\% & 61.6\% & 13.09 \\
OLMo-7B (7B)            & 73.3\% & 67.5\% & 9.59  \\
\midrule
\multicolumn{4}{l}{\textit{Pruning replication only (10 models total)}} \\
Phi-3-mini (3.8B)       & 81.9\% & 73.5\% & 9.48  \\
Mistral-7B-v0.1         & 80.3\% & 75.3\% & 8.08  \\
Meta-Llama-3-8B         & 80.8\% & 73.9\% & 7.44  \\
Llama-3.1-8B-Instr.     & 82.3\% & 73.3\% & 8.92$^\dagger$ \\
Llama-3.2-3B            & 74.5\% & 69.4\% & 9.51  \\
Llama-3.2-1B            & 66.3\% & 60.9\% & 11.96 \\
Qwen2.5-1.5B-Instr.     & 76.8\% & 62.7\% & 12.21 \\
Gemma-2-9B-it           & 85.9\% & 75.9\% & 24.08$^\ddagger$ \\
\bottomrule
\end{tabular}
\begin{minipage}{\linewidth}
\vspace{4pt}
\footnotesize
$^\dagger$~LLama-3.1-8B-Instr.~PPL measured in 8-bit (INT8); ARC-Easy and Winogrande unaffected (float16).
\newline $^\ddagger$~Gemma PPL also uses \texttt{max\_length=512}. See Appendix~\ref{app:pruning} for details.
\end{minipage}
\end{table}

\section{Pruning Replication}
\label{app:pruning}

\begin{table}[h]
\centering
\caption{Effect of zeroing Super Weight coordinates. Each metric shown as baseline\,$\to$\,pruned. \textbf{Top block}: coordinates from \citet{yu2025SuperWeight}, catastrophic degradation confirmed. \textbf{Bottom block}: top-magnitude identification, no effect observed. $^\dagger$~Winogrande near-random (49.3\%) confirms total collapse. $^\ddagger$~PPL diverged to $\infty$.}
\label{tab:pruning}
\small
\begin{tabular}{@{}lrccc@{}}
\toprule
\textbf{Model} & \textbf{SWs} & \textbf{WikiText-2 PPL} & \textbf{ARC-Easy} & \textbf{Winogrande} \\
\midrule
\multicolumn{5}{l}{\textit{Paper coordinates \citep{yu2025SuperWeight}, catastrophic degradation confirmed}} \\
OLMo-1B          & 2 & 13.09\,$\to$\,47{,}951 & 60.6\%\,$\to$\,27.0\% & 61.6\%\,$\to$\,49.3\%$^\dagger$ \\
OLMo-7B          & 2 & 9.59\,$\to$\,42{,}024  & 73.3\%\,$\to$\,60.5\% & 67.5\%\,$\to$\,62.3\% \\
Phi-3-mini       & 6 & 9.48\,$\to$\,3{,}543   & 81.9\%\,$\to$\,34.3\% & 73.5\%\,$\to$\,51.9\% \\
Mistral-7B-v0.1  & 1 & 8.08\,$\to$\,$\infty^\ddagger$ & 80.3\%\,$\to$\,25.6\% & 75.3\%\,$\to$\,49.5\% \\
Meta-Llama-3-8B  & 3 & 7.44\,$\to$\,$\infty^\ddagger$ & 80.8\%\,$\to$\,25.1\% & 73.9\%\,$\to$\,48.2\% \\
\midrule
\multicolumn{5}{l}{\textit{Top-magnitude identification, no effect observed}} \\
Llama-3.1-8B-Instr. & 2 & 8.92\,$\to$\,8.92$^*$   & 82.3\%\,$\to$\,82.2\% & 73.3\%\,$\to$\,73.6\% \\
Llama-3.2-3B        & 2 & 9.51\,$\to$\,9.51         & 74.5\%\,$\to$\,74.9\% & 69.4\%\,$\to$\,69.5\% \\
Llama-3.2-1B        & 2 & 11.96\,$\to$\,11.96       & 66.3\%\,$\to$\,66.3\% & 60.9\%\,$\to$\,60.9\% \\
Qwen2.5-1.5B-Instr. & 2 & 12.21\,$\to$\,12.21       & 76.8\%\,$\to$\,76.8\% & 62.7\%\,$\to$\,63.0\% \\
Gemma-2-9B-it       & 2 & 24.08\,$\to$\,24.04$^{**}$  & 85.9\%\,$\to$\,85.8\% & 75.9\%\,$\to$\,76.0\% \\
\midrule
\multicolumn{5}{p{0.92\linewidth}}{\footnotesize
$^*$~PPL measured in 8-bit (INT8) quantization; Llama-3.1-8B requires 17\,GB in float16, leaving insufficient headroom for wikitext rolling evaluation on a 22\,GB GPU. $^{**}$~Gemma-2-9B additionally required \texttt{max\_length=512}; PPL value not directly comparable to other rows. ARC-Easy and Winogrande for both models were evaluated in float16 and are fully comparable.
} \\
\bottomrule
\end{tabular}
\end{table}

Zeroing the Super Weight coordinates from \citet{yu2025SuperWeight} causes catastrophic degradation in all five models tested with paper coordinates (Table~\ref{tab:pruning}). For OLMo-1B, perplexity increases 3,663$\times$ (13.09\,$\to$\,47,951) and ARC-Easy accuracy collapses to near random-guessing (60.6\%\,$\to$\,27.0\%). For OLMo-7B the PPL collapse is even more severe (4,400$\times$), though the accuracy drop is smaller, consistent with larger model capacity. We confirm the same phenomenon in Phi-3-mini (PPL 9.48\,$\to$\,3,543; ARC 81.9\%\,$\to$\,34.3\%), Mistral-7B-v0.1 (PPL $\to\infty$; ARC 80.3\%\,$\to$\,25.6\%), and Meta-Llama-3-8B (PPL $\to\infty$; ARC 80.8\%\,$\to$\,25.1\%) using the original paper's coordinates.

Applying magnitude-based identification to five additional models (Llama-3.1-8B-Instruct, Llama-3.2-3B, Llama-3.2-1B, Qwen2.5-1.5B-Instruct, and Gemma-2-9B-it) produced no measurable effect when the top-magnitude \texttt{down\_proj} positions were zeroed. PPL changes are $<$0.1\% in all cases, and ARC-Easy and Winogrande accuracy are statistically identical before and after pruning. This is surprising: Qwen2.5-1.5B's top weight magnitudes (1.20) are \emph{higher} than OLMo-1B's (0.83), and Gemma-2-9B-it achieves the highest baseline accuracy of any tested model (85.9\%), yet pruning its top-magnitude positions is harmless. High magnitude is therefore not sufficient to identify Super Weights. The phenomenon depends on the specific activation outlier patterns in each model family, not weight magnitude alone. This is consistent with \citet{yu2025SuperWeight}'s activation-spike-based identification method, which may capture functionally critical positions that magnitude-based ranking misses. Notably, in one case (Llama-3.2-3B), accuracy marginally \emph{increased} after zeroing the top-magnitude positions (74.49\% $\to$ 74.92\%), suggesting those positions may even carry slight negative contributions under our identification method. These findings suggest that \citet{yu2025SuperWeight}'s original evaluation was limited to a narrow set of architectures (primarily LLaMA-family and OLMo), and that the Super Weight phenomenon does not universally apply to all modern LLMs via magnitude-based identification.

\paragraph{PPL evaluation note.} WikiText-2 perplexity for Llama-3.1-8B-Instruct and Gemma-2-9B-it was evaluated using 8-bit (INT8) quantization, as float16 inference requires 17--19\,GB, leaving insufficient memory for the rolling-context wikitext evaluation on our 22\,GB GPUs. INT8 reduces the model footprint to ${\approx}$9--10\,GB while leaving ARC-Easy and Winogrande unaffected (those were evaluated in float16). Gemma-2-9B-it additionally required \texttt{max\_length=512} (vs.\ the standard 2048) due to its logit soft-capping operation inflating peak memory; its PPL value (24.08) is therefore not directly comparable to other models in Table~\ref{tab:pruning} and is provided for completeness. The key finding, that pruning top-magnitude positions has no effect, holds unambiguously for both models across all three metrics.

All subsequent experiments (direct training, neighborhood training, LoRA variants, seed ablation) focus exclusively on OLMo-1B and OLMo-7B, where Super Weight criticality is confirmed. The multi-model pruning results above validate that the phenomenon extends beyond OLMo for models in the original paper's test set; all further analysis uses OLMo as the controlled testbed. We choose OLMo specifically because: (1) Super Weights are confirmed critical via the original paper's activation-spike identification; (2) OLMo is fully open-source with transparent training, enabling reproducible analysis of weight statistics and gradient behaviour; and (3) we have validated results on both a 1B and 7B variant, allowing scale comparisons within the same model family.

\section{Super Weight Stability and Weight Change Analysis}
\label{app:stability}

\begin{table}[h]
\centering
\caption{Super Weight properties before and after vanilla LoRA training (OLMo-1B).}
\label{tab:sw_stability}
\small
\begin{tabular}{@{}lccc@{}}
\toprule
\textbf{Metric} & \textbf{Before} & \textbf{After} & \textbf{Change} \\
\midrule
Mean SW magnitude & 0.206 & 0.206 & $-$0.03\% \\
SW percentile rank & 100.0\% & 100.0\% & 0\% \\
SW as local maximum & 100\% & 100\% & 0\% \\
SW/neighbor ratio & 58.4$\times$ & 58.6$\times$ & +0.3\% \\
\bottomrule
\end{tabular}
\end{table}

Super Weights remain stable under LoRA training: all maintain their top-1\% percentile rank, all remain local maxima, and their magnitude changes by only $-$0.03\% on average (Table~\ref{tab:sw_stability}).

\begin{table}[h]
\centering
\caption{Weight change statistics at \texttt{down\_proj} Super Weight coordinates across methods (OLMo-1B). For LoRA, changes are measured at the same $(i,j)$ coordinates in the attention $\Delta W$ matrix, not in \texttt{down\_proj} itself (LoRA does not modify \texttt{down\_proj}).}
\label{tab:weight_changes}
\small
\begin{tabular}{@{}lrrrr@{}}
\toprule
\textbf{Method} & \textbf{Mean $\Delta$} & \textbf{Median $\Delta$} & \textbf{Max $\Delta$} & \textbf{CV} \\
\midrule
Pruning        & 100\%  & 100\%  & 100\%   & 0.00 \\
Direct Training & 8.19\% & 7.55\% & 30.89\% & 0.66 \\
Vanilla LoRA   & 0.15\% & 0.10\% & 5.56\%  & 0.94 \\
\bottomrule
\end{tabular}
\end{table}

Direct training produces changes 55$\times$ larger than LoRA (8.19\% vs 0.15\% average), yet LoRA succeeds while direct training does not (Table~\ref{tab:weight_changes}). The size of the update alone does not determine outcome; coordinated full-layer structure does.

\section{Seed Ablation}
\label{app:seed_ablation}

\begin{table}[h]
\centering
\caption{10-seed robustness check: vanilla LoRA vs.\ LoRA-dproj-SW-freeze ($s$=0.0), repeating the extreme case of Table~\ref{tab:superlora} across seeds 42--51 to test whether the single-seed gap reflects genuine sensitivity or seed variance.}
\label{tab:seed_ablation}
\small
\begin{tabular}{@{}lrrrr@{}}
\toprule
\textbf{Method} & \textbf{Mean Acc} & \textbf{Std Dev} & \textbf{Identical Seeds} & \textbf{p-value} \\
\midrule
\multicolumn{5}{l}{\textit{OLMo-1B}} \\
Vanilla LoRA     & 62.90\% & 0.45\% & -- & -- \\
LoRA-dproj-SW-freeze & 62.90\% & 0.45\% & 8/10 (80\%) & $>$0.05 \\
\midrule
\multicolumn{5}{l}{\textit{OLMo-7B}} \\
Vanilla LoRA     & 66.05\% & 1.14\% & -- & -- \\
LoRA-dproj-SW-freeze & 66.05\% & 1.14\% & 10/10 (100\%) & $=1.0$ \\
\bottomrule
\end{tabular}
\end{table}

Table~\ref{tab:seed_ablation} distinguishes two related but separate experiments. Table~\ref{tab:superlora} (LoRA-dproj-SW-freeze) tests a range of scale factors $s$ applied to attention positions whose indices match Super Weight coordinates in \texttt{down\_proj}. It varies \emph{how much} those positions are restricted. Table~\ref{tab:seed_ablation} (seed ablation) asks whether the single-seed gap between $s=1.0$ and $s=0.0$ is real or noise, by repeating both conditions across 10 seeds. The result is clear: vanilla LoRA and LoRA-dproj-SW-freeze achieve identical mean accuracy (62.90\% $\pm$ 0.45\%) across 10 seeds on OLMo-1B, with 8/10 seeds producing bit-for-bit identical predictions and no significant difference ($p > 0.05$). Restricting the attention update at positions corresponding to Super Weight coordinates, which freezes 6.8\% of LoRA parameters for OLMo-1B and just 0.5\% for OLMo-7B, has no measurable effect on performance in either model. For OLMo-7B, all 10 seeds produce bit-for-bit identical predictions between vanilla LoRA and LoRA-dproj-SW-freeze (p$=1.0$), a stronger result than OLMo-1B (8/10 seeds identical, $p > 0.05$). LoRA's full-layer low-rank structure compensates robustly: the MLP's dense transformation means every attention output dimension influences every Super Weight downstream, so the 93--99.5\% of unrestricted LoRA parameters are more than sufficient to route representations appropriately through the fixed Super Weights. This robustness is further confirmed by the LoRA-$\Delta W$-SW-freeze experiment (Table~\ref{tab:lora_sw_freeze}): even freezing the top-1,000 highest-magnitude positions in LoRA's own learned update $\Delta W$ has no measurable effect. Base model Super Weights are individually catastrophic when removed; LoRA's highest-magnitude positions are not. This asymmetry reflects a fundamental difference between pretrained weight structure and learned adaptation structure.

\begin{figure}[h]
\centering
\includegraphics[width=0.85\textwidth]{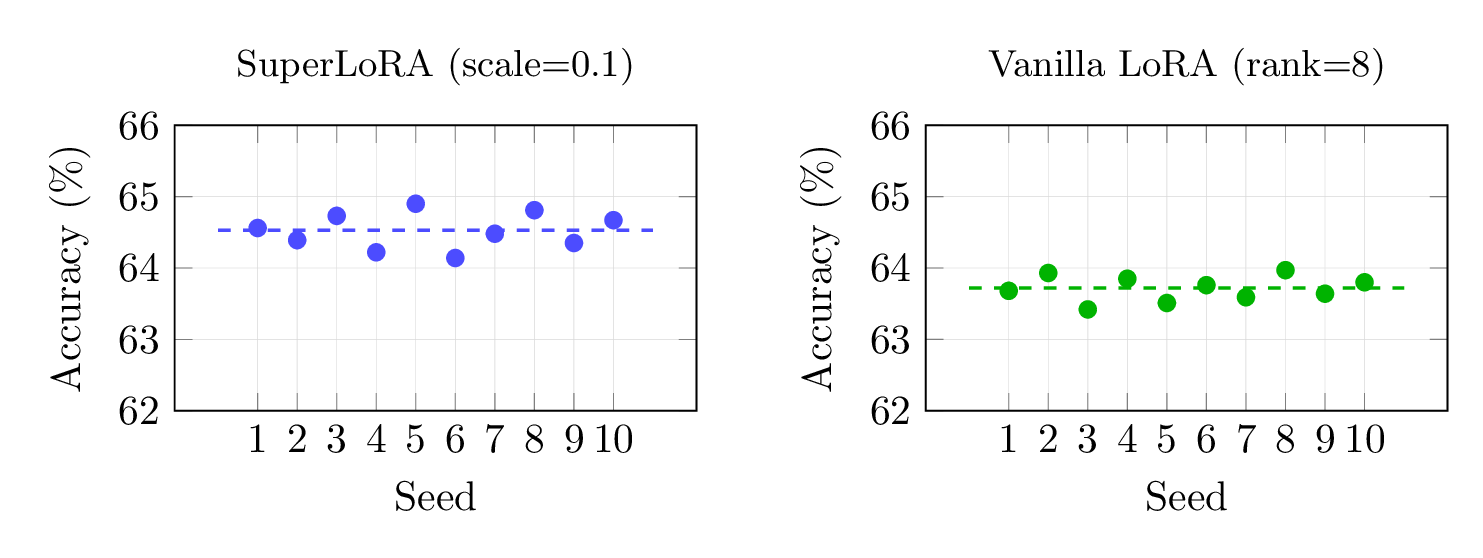}
\caption{Per-seed accuracy for vanilla LoRA and LoRA-dproj-SW-freeze (OLMo-1B, seeds 42--51). 8 of 10 seeds produce identical predictions, confirming the single-seed gap falls within seed variation.}
\label{fig:seed_ablation}
\end{figure}

\section{Winogrande Results}
\label{app:winogrande}

We report OLMo-1B Winogrande results for the training methods in Table~\ref{tab:winogrande}. Winogrande behaves differently from ARC-Easy, and the difference is informative. The pretrained model already scores 61.6\% on this two-choice commonsense task, and no training method improves on it, because there is little task-specific signal to gain from the small training set and the risk of overfitting is high. What the two method families still separate on is the \emph{direction and severity} of the change. Training at Super Weight coordinates and their neighborhoods collapses to roughly 50\%, which is chance for a two-choice task, matching the collapse we see on ARC-Easy. LoRA, by contrast, stays within a few points of the pretrained baseline. The absolute picture is a task where training does not help, but the relative picture is the same as ARC-Easy: training at Super Weight coordinates destroys the capability, while full-layer low-rank updates preserve it.

\begin{table}[h]
\centering
\caption{OLMo-1B Winogrande accuracy (two-choice; chance is 50\%). No method improves over the 61.6\% pretrained baseline, but training at Super Weight coordinates collapses to chance while LoRA stays close to baseline.}
\label{tab:winogrande}
\small
\begin{tabular}{@{}lrr@{}}
\toprule
\textbf{Method} & \textbf{Params} & \textbf{Winogrande} \\
\midrule
Pretrained baseline           & ---   & 61.6\% \\
\midrule
Direct SW ($k$=100)           & 100   & 53.2\% \\
Direct SW ($k$=1,000)         & 1,000 & 50.7\% \\
Neighborhood ($k$=4,096)      & 36,864 & 50.6\% \\
\midrule
LoRA ($r$=1)                  & 4,096 & 59.6\% \\
LoRA ($r$=2)                  & 8,192 & 58.6\% \\
\bottomrule
\end{tabular}
\end{table}

\section{Structured vs.\ Unstructured Freezing}
\label{app:obs4}

Recent work by \citet{morris2026tinylora} shows that full-layer low-rank updates can succeed with as few as 13 trainable parameters using reinforcement learning on math reasoning tasks. This appears to contradict our finding that training at Super Weight coordinates fails with 100--8,192 parameters, but the comparison is informative.

TinyLoRA uses a weight-tying mechanism: $\Delta W = U \cdot (\sum v_i P_i) \cdot V^{\top}$, where $U, \Sigma, V$ are frozen SVD components and $P_i$ are fixed random projections. Only the vector $\mathbf{v}$ (13 parameters) is trainable. The frozen structure propagates these 13 parameters across all positions in the weight matrix, so each position $(i,j)$ receives a coordinated update through the frozen low-rank basis. This is structured freezing: the frozen components maintain relational structure across the layer.

Our sparse Super Weight training performs unstructured freezing: selected positions receive gradient updates while neighboring positions stay exactly at pretrained values, with no mechanism to coordinate between them. The relevant distinction is whether the frozen components maintain layer-wide coordination. Both vanilla LoRA (2.1M trainable parameters) and TinyLoRA (13 trainable parameters) succeed because both produce coordinated updates across the full layer.

\section{LoRA-dproj-SW-freeze Geometry}
\label{app:lora_geometry}

It is important to understand the geometry of this experiment before reading the results. Super Weights reside in \texttt{down\_proj} (shape 2048$\times$8192 in OLMo-1B), while LoRA targets attention projections (\texttt{q/k/v/o\_proj}, shape 2048$\times$2048). These are entirely different weight matrices. LoRA never touches Super Weight values directly.

The connection runs through the residual stream. A Super Weight at position $(r, c)$ in \texttt{down\_proj} amplifies MLP hidden dimension $c$ into output dimension $r$. Output dimension $r$ of the attention block contributes to the residual that feeds into the MLP. When LoRA modifies \texttt{o\_proj} row $r$, it shapes attention output dimension $r$, which flows through the residual into the MLP and ultimately influences what the Super Weight at $(r, c)$ receives. The LoRA-dproj-SW-freeze experiment restricts this: we scale down LoRA's update at attention positions $(i,j)$ whose indices match Super Weight coordinates. Specifically, freezing row $r$ of \texttt{lora\_B} in \texttt{o\_proj} prevents LoRA from learning to adjust attention output dimension $r$, which is the same output dimension that the Super Weight writes into. At $s=0.0$, this affects 143,648 of 2,097,152 LoRA parameters for OLMo-1B (6.8\%), and 39,520 of 8,388,608 for OLMo-7B (0.5\%). The difference reflects a notable property of 7B Super Weights: they concentrate in far fewer unique output rows (573 unique rows vs.\ 1,889 for 1B), meaning the same top-10,000 Super Weights in a larger model affect a smaller fraction of the LoRA parameter space. This is the mechanism connecting the two experiments.

\section{Theoretical Analysis: Intrinsic Dimensionality Framework}
\label{app:theory_intrinsic}

Our experiments reveal a graded pattern. Training at Super Weight coordinates collapses to random guessing, sparse training at random coordinates in the same layers is suboptimal but stays above baseline, and LoRA's full-layer low-rank updates succeed. We develop a theoretical framework that explains this pattern. The analysis rests on three pillars: (1)~sparse coordinate-aligned updates capture only a vanishing fraction of the low-dimensional subspace in which fine-tuning operates, which makes them suboptimal in general; (2)~low-rank factorizations span this subspace with far fewer free parameters; and (3)~the low-rank structure implicitly preserves large-magnitude weights by distributing perturbation inversely with pretrained magnitude, while training directed at those large-magnitude Super Weight positions instead triggers a gradient-amplification collapse.

\subsection{Why Sparse Training Fails: Intrinsic Dimensionality and Coordinate Alignment}
\label{sec:theory_sparse}

\citet{aghajanyan2021intrinsic} showed that fine-tuning operates in a subspace of surprisingly low \emph{intrinsic dimensionality}. They reparameterize the fine-tuning objective as $\theta = \theta_0 + P\theta_d$, where $\theta_0 \in \mathbb{R}^N$ is the pretrained parameter vector, $P \in \mathbb{R}^{N \times d}$ is a fixed random dense projection (implemented via the Fastfood transform), and $\theta_d \in \mathbb{R}^d$ is the only optimized variable. They define $d_{90}$ as the smallest $d$ achieving 90\% of full fine-tuning performance. For RoBERTa-base ($N \approx 125$M), $d_{90}$ ranges from ${\sim}200$ (MRPC) to ${\sim}1{,}600$ (MNLI), establishing that the effective fine-tuning subspace $\mathcal{S} \subset \mathbb{R}^N$ has dimension $d \ll N$.

This result has direct implications for sparse training. Sparse-subset training of $k$ parameters restricts weight updates to the coordinate-aligned subspace
\begin{equation}
\mathcal{V}_{\text{sparse}} = \operatorname{span}\{e_{(i_1,j_1)}, \ldots, e_{(i_k,j_k)}\} \subset \mathbb{R}^N,
\label{eq:sparse_subspace}
\end{equation}
where $e_{(i,j)}$ denotes the standard basis vector corresponding to position $(i,j)$ in the vectorized weight matrix. For sparse training to succeed, $\mathcal{V}_{\text{sparse}}$ must have substantial overlap with~$\mathcal{S}$.

\begin{proposition}[Sparse training and subspace coverage]
\label{prop:sparse_fail}
 Let $\mathcal{S}$ be a $d$-dimensional subspace of $\mathbb{R}^N$ drawn uniformly from the Grassmannian $\mathrm{Gr}(d, N)$, and let $\mathcal{V}_{\mathrm{sparse}}$ be a fixed $k$-dimensional coordinate-aligned subspace as in \eqref{eq:sparse_subspace}. Then the expected fraction of $\mathcal{S}$'s variance captured by projection onto $\mathcal{V}_{\mathrm{sparse}}$ is
\begin{equation}
\mathbb{E}\left[\frac{\|\Pi_{\mathcal{V}_{\mathrm{sparse}}} \Pi_{\mathcal{S}}\|_F^2}{d}\right] = \frac{k}{N},
\label{eq:sparse_coverage}
\end{equation}
where $\Pi_{\mathcal{V}}$ denotes orthogonal projection onto $\mathcal{V}$.
\end{proposition}

\begin{proof}
Let $U \in \mathbb{R}^{N \times d}$ be an orthonormal basis for $\mathcal{S}$, drawn uniformly from the Stiefel manifold. Then $\|\Pi_{\mathcal{V}_{\mathrm{sparse}}} \Pi_{\mathcal{S}}\|_F^2 = \sum_{l=1}^{k} \|U_{i_l,:}\|^2$, where $U_{i_l,:}$ selects the rows of $U$ corresponding to the $k$ chosen coordinates. By the rotational invariance of the Haar measure on $\mathrm{Gr}(d,N)$, each row of $U$ has expected squared norm $d/N$. Summing over $k$ rows gives $\mathbb{E}[\sum_{l=1}^k \|U_{i_l,:}\|^2] = kd/N$. Dividing by $d$ yields $k/N$.
\end{proof}

For OLMo-1B with $N \approx 1.28 \times 10^9$ parameters, even $k = 8{,}192$ sparse coordinates capture an expected fraction of $6.4 \times 10^{-6}$ of the fine-tuning subspace, effectively zero. This holds regardless of \emph{which} coordinates are selected (Super Weights, magnitude-ranked, or random), because the result depends only on $k/N$, not on the specific positions. Proposition~\ref{prop:sparse_fail} thus predicts that \emph{all} sparse training at $k \ll N$ is limited to a vanishing slice of the fine-tuning subspace and is therefore suboptimal relative to full-layer low-rank updates, whichever coordinates it targets. This is a prediction of \emph{suboptimality}, not of collapse. It is consistent with our random-position control, which trains $k=4{,}096$ random \texttt{down\_proj} coordinates and lands at 64.18\%, above the 60.65\% baseline but below LoRA. The catastrophic collapse we observe at Super Weight coordinates is therefore not explained by the subspace-coverage argument alone, and we turn to a position-specific mechanism next.

\paragraph{Why Super Weight positions fail catastrophically.}
The subspace argument predicts that sparse training will \emph{underperform}, but our experiments show something stronger: training at Super Weight positions causes training loss to decrease while validation perplexity explodes (Table~\ref{tab:direct_training}). We attribute this to gradient amplification at outlier positions. In the MLP forward pass $\mathbf{y} = W_{\text{down}} \, \sigma(W_{\text{up}} \mathbf{x})$, a Super Weight $W_{i^*j^*}$ with magnitude $M \gg \bar{w}$ (where $\bar{w}$ is the typical entry magnitude) produces an output contribution
\begin{equation}
y_{i^*} \ni W_{i^*j^*} \cdot h_{j^*} = M \cdot h_{j^*},
\label{eq:forward_amplification}
\end{equation}
where $h_{j^*} = \sigma(W_{\text{up}} \mathbf{x})_{j^*}$. When $W_{i^*j^*}$ is the only trainable parameter, the gradient $\partial \mathcal{L}/\partial W_{i^*j^*} = \delta_{i^*} \cdot h_{j^*}$ (where $\delta_{i^*}$ is the error signal at output dimension $i^*$) is sensitive to $W_{i^*j^*}$ through the dependence of $\delta_{i^*}$ on $y_{i^*}$, creating a positive feedback loop: large weight $\to$ large activation $\to$ large gradient $\to$ large update $\to$ increased divergence. This explains why Super Weight training doesn't merely underperform (as any sparse training would by Proposition~\ref{prop:sparse_fail}) but collapses catastrophically.

\begin{remark}
We frame gradient amplification as a mechanistic explanation rather than a formal bound, because formalizing the feedback loop requires assumptions about loss landscape curvature at Super Weight positions that we cannot verify without additional experiments. The intrinsic dimensionality argument (Proposition~\ref{prop:sparse_fail}) provides the rigorous foundation and predicts that sparse training is suboptimal; gradient amplification explains the \emph{severity} of collapse at Super Weight positions specifically. Our random-position control disentangles the two effects: at matched $k=4{,}096$, random \texttt{down\_proj} coordinates degrade gracefully to 64.18\% (above baseline), while Super Weight coordinates collapse to chance. This matches the prediction that random sparse training loses the subspace-coverage advantage of LoRA yet avoids the destabilizing feedback loop that only the outlier positions trigger.
\end{remark}

\subsection{Why LoRA Succeeds: Low-Rank Subspace Coverage}
\label{sec:theory_lora}

LoRA parameterizes the weight update as $\Delta W = BA \cdot \alpha/r$, where $B \in \mathbb{R}^{m \times r}$ and $A \in \mathbb{R}^{r \times n}$. The set of matrices of rank at most $r$ in $\mathbb{R}^{m \times n}$ forms a smooth manifold $\mathcal{M}_r$ of dimension
\begin{equation}
\dim(\mathcal{M}_r) = r(m + n - r).
\label{eq:manifold_dim}
\end{equation}

For OLMo-1B attention projections ($m = n = 2048$, $r = 8$), this gives $\dim(\mathcal{M}_r) = 8 \times (2048 + 2048 - 8) = 32{,}704$. LoRA is applied to all four attention projections in each of 16 layers, yielding a total tangent space dimension of $16 \times 4 \times 32{,}704 = 2{,}093{,}056$ across the model, well above the intrinsic dimensionality $d_{90} \lesssim 1{,}600$ reported by \citet{aghajanyan2021intrinsic}.

\begin{proposition}[Low-rank vs.\ sparse expressiveness]
\label{prop:lora_coverage}
The rank-$r$ manifold $\mathcal{M}_r \subset \mathbb{R}^{m \times n}$ has dimension $r(m+n-r)$, while a $k$-sparse coordinate subspace has dimension $k$. For $r = 8$, $m = n = 2048$:
\begin{equation}
\frac{\dim(\mathcal{M}_r)}{k} = \frac{32{,}704}{k}.
\end{equation}
At $k = 8{,}192$ (the largest sparse experiment), the per-layer ratio is $4\times$. Across all LoRA-targeted layers, the total ratio is $2{,}093{,}056 / 8{,}192 \approx 256\times$.
\end{proposition}

The distinction goes beyond dimensionality counting. Each basis direction in $\mathcal{V}_{\text{sparse}}$ modifies a single matrix entry, leaving all others unchanged. Each basis direction in $\mathcal{M}_r$ is a rank-1 matrix $\mathbf{b}\mathbf{a}^\top$ that modifies \emph{every} entry simultaneously through a dense outer product. A rank-1 perturbation $\mathbf{b}\mathbf{a}^\top$ has $mn$ nonzero entries in general position, while a sparse basis vector $e_{(i,j)}$ has exactly one. This dense coupling is what the intrinsic dimensionality framework requires: the random projection $P$ in \citet{aghajanyan2021intrinsic} succeeds because it maps each low-dimensional coordinate into a \emph{dense} direction in parameter space, and LoRA's rank-1 outer products provide the same dense coverage.

\paragraph{Connection to the intrinsic dimensionality framework.}
LoRA with Kaiming/Gaussian initialization of $A$ produces random dense directions in weight space at initialization, analogous to the Fastfood projections in \citet{aghajanyan2021intrinsic}. As training progresses, gradient descent steers these directions toward the fine-tuning subspace $\mathcal{S}$. The rank constraint $r \ll \min(m,n)$ limits the update to an $r(m+n-r)$-dimensional manifold, but this is sufficient when $d_{90} \ll r(m+n-r)$, which holds for all tasks studied by \citet{aghajanyan2021intrinsic}.

\subsection{Implicit Preservation of Super Weights Under Low-Rank Updates}
\label{sec:theory_preservation}

We now show that LoRA's low-rank structure implicitly protects large-magnitude weights from large \emph{relative} perturbation, without any explicit constraint.

\begin{proposition}[Entry-wise perturbation bound]
\label{prop:entry_bound}
Let $\Delta W = BA \cdot \alpha/r$ where $B \in \mathbb{R}^{m \times r}$, $A \in \mathbb{R}^{r \times n}$. For any position $(i,j)$:
\begin{equation}
|\Delta W_{ij}| = \frac{\alpha}{r} \left| \sum_{k=1}^{r} B_{ik} A_{kj} \right| \leq \frac{\alpha}{r} \|B_{i,:}\|_2 \, \|A_{:,j}\|_2.
\label{eq:entry_bound}
\end{equation}
The relative perturbation at position $(i,j)$ with pretrained weight $W_{ij} \neq 0$ satisfies
\begin{equation}
\frac{|\Delta W_{ij}|}{|W_{ij}|} \leq \frac{\alpha}{r} \cdot \frac{\|B_{i,:}\|_2 \, \|A_{:,j}\|_2}{|W_{ij}|},
\label{eq:relative_bound}
\end{equation}
which is monotonically decreasing in $|W_{ij}|$ for fixed LoRA factors.
\end{proposition}

The worst-case bound \eqref{eq:relative_bound} is loose for typical matrices. We supplement it with a typical-case analysis.

\begin{proposition}[Typical-case relative perturbation]
\label{prop:typical_case}
If the learned LoRA factors have approximately independent entries with $B_{ik} \sim \mathcal{N}(0, \sigma_B^2)$ and $A_{kj} \sim \mathcal{N}(0, \sigma_A^2)$, then $\Delta W_{ij} = (\alpha/r) \sum_{k=1}^r B_{ik} A_{kj}$ is a sum of $r$ independent sub-exponential random variables with
\begin{align}
\mathbb{E}[|\Delta W_{ij}|] &= \frac{\alpha}{r} \cdot r \cdot \sigma_B \sigma_A \sqrt{\frac{2}{\pi}} = \alpha \sigma_B \sigma_A \sqrt{\frac{2}{\pi}}, \label{eq:typical_mean} \\
\mathrm{Var}(\Delta W_{ij}) &= \frac{\alpha^2}{r^2} \cdot r \cdot \sigma_B^2 \sigma_A^2 = \frac{\alpha^2 \sigma_B^2 \sigma_A^2}{r}. \label{eq:typical_var}
\end{align}
The typical relative perturbation at a position with pretrained magnitude $|W_{ij}|$ is therefore
\begin{equation}
\frac{|\Delta W_{ij}|}{|W_{ij}|} \approx \frac{\alpha \sigma_B \sigma_A \sqrt{2/\pi}}{|W_{ij}|} \pm O\!\left(\frac{\alpha \sigma_B \sigma_A}{|W_{ij}|\sqrt{r}}\right).
\label{eq:typical_relative}
\end{equation}
\end{proposition}

\begin{proof}
Each term $Z_k = B_{ik} A_{kj}$ is a product of independent Gaussians with $\mathbb{E}[Z_k] = 0$, $\mathbb{E}[|Z_k|] = \sigma_B \sigma_A \cdot {2}/{\pi}$, and $\mathrm{Var}(Z_k) = \sigma_B^2 \sigma_A^2$. The sum $S = \sum_{k=1}^r Z_k$ has $\mathrm{Var}(S) = r \sigma_B^2 \sigma_A^2$. Scaling by $\alpha/r$ gives \eqref{eq:typical_var}. The mean of $|\Delta W_{ij}|$ follows from $\mathbb{E}[|S|] = \sqrt{2r/\pi} \cdot \sigma_B \sigma_A$, scaled by $\alpha/r$.
\end{proof}

Equation~\eqref{eq:typical_relative} shows that the typical relative perturbation is $O(1/|W_{ij}|)$. Super Weights in OLMo-1B have magnitudes ${\sim}58\times$ larger than their neighbors (Table~\ref{tab:sw_stability}), so they receive proportionally smaller relative perturbation. This is consistent with the empirically observed 0.15\% mean relative change at Super Weight positions under LoRA (Table~\ref{tab:weight_changes}).

\paragraph{Implicit regularization through weight decay.}
AdamW applies weight decay independently to $B$ and $A$, penalizing $\|B\|_F^2 + \|A\|_F^2$. For the factorization $\Delta W = BA$, it is known that \citep{srebro2004maximum}:
\begin{equation}
\min_{BA = \Delta W} \frac{1}{2}\left(\|B\|_F^2 + \|A\|_F^2\right) = \|\Delta W\|_*,
\label{eq:nuclear_norm}
\end{equation}
where $\|\cdot\|_*$ denotes the nuclear norm (sum of singular values). Weight decay on the factors thus biases LoRA toward low nuclear norm solutions. Since $\|\Delta W\|_* \geq \|\Delta W\|_2 \geq |\Delta W_{ij}|$ for all $(i,j)$, controlling the nuclear norm simultaneously limits all entry-wise perturbations. This provides a complementary mechanism to Proposition~\ref{prop:entry_bound}: weight decay keeps $\sigma_B$ and $\sigma_A$ small, which keeps all $|\Delta W_{ij}|$ small, with the largest-magnitude pretrained weights receiving the smallest relative change.

\subsection{Testable Predictions}
\label{sec:predictions}

The framework yields concrete predictions, each mapped to experimental results:

\begin{enumerate}
\item \textbf{Sparse training is suboptimal for any coordinate choice, but only Super Weight coordinates collapse.} Proposition~\ref{prop:sparse_fail} predicts that \emph{any} selection of $k \ll N$ coordinate-aligned positions covers only a $k/N \approx 0$ fraction of the fine-tuning subspace and is therefore suboptimal relative to full-layer low-rank updates. It does not predict collapse. Our random-position control confirms both halves of this: at matched $k=4{,}096$, random \texttt{down\_proj} coordinates reach 64.18\%, above the 60.65\% baseline but below LoRA's 66.88\%, so sparse training is suboptimal yet functional, whereas the same-size update at Super Weight coordinates collapses to chance (Tables~\ref{tab:direct_training}--\ref{tab:neighborhood}). The collapse is thus attributable to the gradient-amplification mechanism at outlier positions, not to subspace coverage.

\item \textbf{LoRA's success should be insensitive to which positions are constrained.} Since LoRA's rank-$r$ manifold is vastly over-parameterized relative to the intrinsic dimension ($32{,}704 \gg d_{90}$), freezing a small fraction of positions should have negligible effect. This is confirmed by LoRA-dproj-SW-freeze (Table~\ref{tab:superlora}), where freezing 6.8\% of LoRA parameters produces statistically indistinguishable results ($p > 0.05$; Table~\ref{tab:seed_ablation}).

\item \textbf{LoRA's highest-magnitude update positions should not be individually critical.} Unlike pretrained Super Weights (which are structural bottlenecks), LoRA's learned $\Delta W$ distributes information across all positions through the low-rank factorization. No single position is a bottleneck. This is confirmed by LoRA-$\Delta W$-SW-freeze (Table~\ref{tab:lora_sw_freeze}): freezing the top-1,000 $\Delta W$ positions has no measurable effect.

\item \textbf{Relative perturbation should decrease with pretrained magnitude.} Proposition~\ref{prop:typical_case} predicts an inverse relationship between $|W_{ij}|$ and relative perturbation $|\Delta W_{ij}|/|W_{ij}|$. Our weight change analysis (Table~\ref{tab:weight_changes}) confirms this: Super Weight positions change by only 0.15\% under LoRA, consistent with the $O(1/|W_{ij}|)$ scaling.
\end{enumerate}

\section{Theoretical Analysis: Implicit Regularization at Super Weight Positions}
\label{app:theory_optimizer}

\citet{an2025systematic} establish that weight outliers and activation outliers
in transformers co-occur with 100\% consistency in the feature dimension, with
amplification ratios $M = |x_c|/\mu_x \sim 10^3$. We show that this coupling
causes adaptive optimizers to implicitly suppress updates at Super Weight
positions, making explicit freezing redundant.

\subsection{Single-layer case}

Let $\mathbf{y} = W\mathbf{x}$ with loss $L(\mathbf{y})$, and let
$w^* = W_{r,c}$ be a Super Weight whose coupled activation satisfies
$|x_c| = M \mu_x$ for $M \gg 1$, where $\mu_x = \frac{1}{n}\sum_j |x_j|$.
We analyze a single input; the argument holds for each training example
independently, and averaging over a batch preserves the suppression since
$M \gg 1$ at every input~\citep{an2025systematic}.

The gradient at position $(r,j)$ is:
\begin{equation}
g_{r,j} = \frac{\partial L}{\partial W_{r,j}} = \frac{\partial L}{\partial y_r} \cdot x_j
\label{eq:gradient}
\end{equation}

The scalar $\partial L / \partial y_r$ is shared across all positions in row $r$.
At the Super Weight position, $g_{r,c} = (\partial L / \partial y_r) \cdot M\mu_x$.
At a typical position where $|x_j| \approx \mu_x$,
$g_{r,j} \approx (\partial L / \partial y_r) \cdot \mu_x$.
The gradient at the Super Weight position is $M$ times larger.

The diagonal Hessian at position $(r,j)$ is (using the fact that $y_r$ is
linear in $W_{r,j}$, so $\partial^2 y_r / \partial W_{r,j}^2 = 0$):
\begin{equation}
H_{r,j} = \frac{\partial^2 L}{\partial W_{r,j}^2}
= \frac{\partial^2 L}{\partial y_r^2} \cdot x_j^2
\label{eq:hessian}
\end{equation}

At the Super Weight position, $H_{r,c} = (\partial^2 L / \partial y_r^2) \cdot M^2 \mu_x^2$.
The curvature is $M^2$ times larger than at a typical position.

\begin{theorem}[Implicit suppression]
\label{thm:suppression}
Under a diagonal second-order update rule $\Delta w_{r,j} = -\eta \, g_{r,j} / H_{r,j}$,
the parameter update at the Super Weight position relative to a typical position
in the same row satisfies:
\begin{equation}
\frac{|\Delta w^*|}{|\Delta w_{\mathrm{typ}}|}
= \frac{|g_{r,c}| / H_{r,c}}{|g_{r,j}| / H_{r,j}}
= \frac{M\mu_x \,/\, M^2\mu_x^2}{\mu_x \,/\, \mu_x^2}
= \frac{1}{M}
\label{eq:suppression}
\end{equation}
Despite receiving an $M$-fold larger gradient, the Super Weight receives a
$1/M$-fold smaller parameter update due to $M^2$-fold larger curvature.
The induced output changes are equalized:
\begin{equation}
|\Delta w^*| \cdot |x_c|
= \frac{|\Delta w_{\mathrm{typ}}|}{M} \cdot M\mu_x
= |\Delta w_{\mathrm{typ}}| \cdot \mu_x
\approx |\Delta w_{\mathrm{typ}}| \cdot |x_j|
\label{eq:output_equal}
\end{equation}
The optimizer suppresses the parameter update by $1/M$ to compensate
for the $M$-fold activation amplification, equalizing output contributions
in magnitude across positions.
\end{theorem}

\begin{proof}
Equation~\eqref{eq:gradient} gives $|g_{r,c}| = |\partial L / \partial y_r| \cdot M\mu_x$
and $|g_{r,j}| = |\partial L / \partial y_r| \cdot \mu_x$. Equation~\eqref{eq:hessian} gives
$H_{r,c} = (\partial^2 L / \partial y_r^2) \cdot M^2\mu_x^2$ and
$H_{r,j} = (\partial^2 L / \partial y_r^2) \cdot \mu_x^2$.
The upstream terms $\partial L / \partial y_r$ and $\partial^2 L / \partial y_r^2$
are shared across all positions in row $r$ and cancel in the ratio.
Substituting into $\Delta w = -\eta\, g/H$ yields Equation~\eqref{eq:suppression}.
Equation~\eqref{eq:output_equal} follows by multiplying each side of
Equation~\eqref{eq:suppression} by the corresponding activation magnitude.
\end{proof}

\begin{corollary}[Redundancy of explicit freezing]
\label{cor:freezing}
Freezing a Super Weight position removes a parameter update of magnitude
$|\Delta w_{\mathrm{typ}}|/M$ from the optimization. The induced output perturbation
equals that of freezing one typical weight, a fraction $O(1/n)$ of the total
layer update for a layer with $n$ columns. Explicit constraints at Super Weight
positions are therefore redundant with the optimizer's implicit regularization.
This holds for the optimizer's natural update trajectory; externally imposed
large perturbations (as in direct Super Weight training) override this
implicit protection.
\end{corollary}

\paragraph{Practical optimizers.}
Theorem~\ref{thm:suppression} assumes a diagonal second-order update.
Adam approximates this with denominator $\sqrt{v_t}$, where $v_t$
accumulates squared gradients; since $\sqrt{v_t}$ scales with gradient
magnitude rather than curvature, Adam yields updates closer to $O(1)$
at all positions. However, the loss curvature at Super Weight
positions is $M^2$ times larger regardless of the optimizer.
A second-order Taylor expansion of the loss around $w^*$ gives:
\begin{equation}
L(w^* + \Delta w) \approx L(w^*) + g_{r,c}\,\Delta w + \tfrac{1}{2}\,H_{r,c}\,(\Delta w)^2
\end{equation}
An update $\Delta w$ at the Super Weight position incurs a curvature penalty
$\frac{1}{2} H_{r,c} (\Delta w)^2 \propto M^2 (\Delta w)^2$, which is $M^2$ times
the penalty for the same-magnitude update at a typical position. Even when
Adam does not fully replicate second-order scaling, overshooting at Super Weight
positions is penalized quadratically more by the loss, constraining the
effective step size through the optimization trajectory.

\paragraph{Connection to Lipschitz stability.}
The $1/M$ suppression is not merely an optimizer convenience but a
structural necessity. \citet{furuya2025approximation} show that for a
gradient-descent-type layer $F_\xi(x) = x - \tau W^\top \sigma(Wx + b)$
to remain $1$-Lipschitz, the step size must satisfy
$\tau \leq 2 / \|W\|_2^2$. For their attention layers, the analogous
constraint is $\eta \leq 2 / \sup_{y \in \Omega} \|Ay\|_2^2$: the
permissible step size is inversely proportional to the squared activation
magnitude. This mirrors our result exactly. At a Super Weight position
where the coupled activation is $M$ times larger, the Hessian-scaled
update is $1/M$ smaller, the same inverse-square relationship that
\citet{furuya2025approximation} derive as a \emph{necessary condition}
for Lipschitz continuity. The pretrained model is approximately Lipschitz
(it generalizes stably across inputs), so any optimizer that preserves
this property must suppress updates at high-activation positions.
Theorem~\ref{thm:suppression} shows that second-order-aware optimizers
achieve this automatically.

\begin{figure}[h]
\centering
\includegraphics[width=0.80\textwidth]{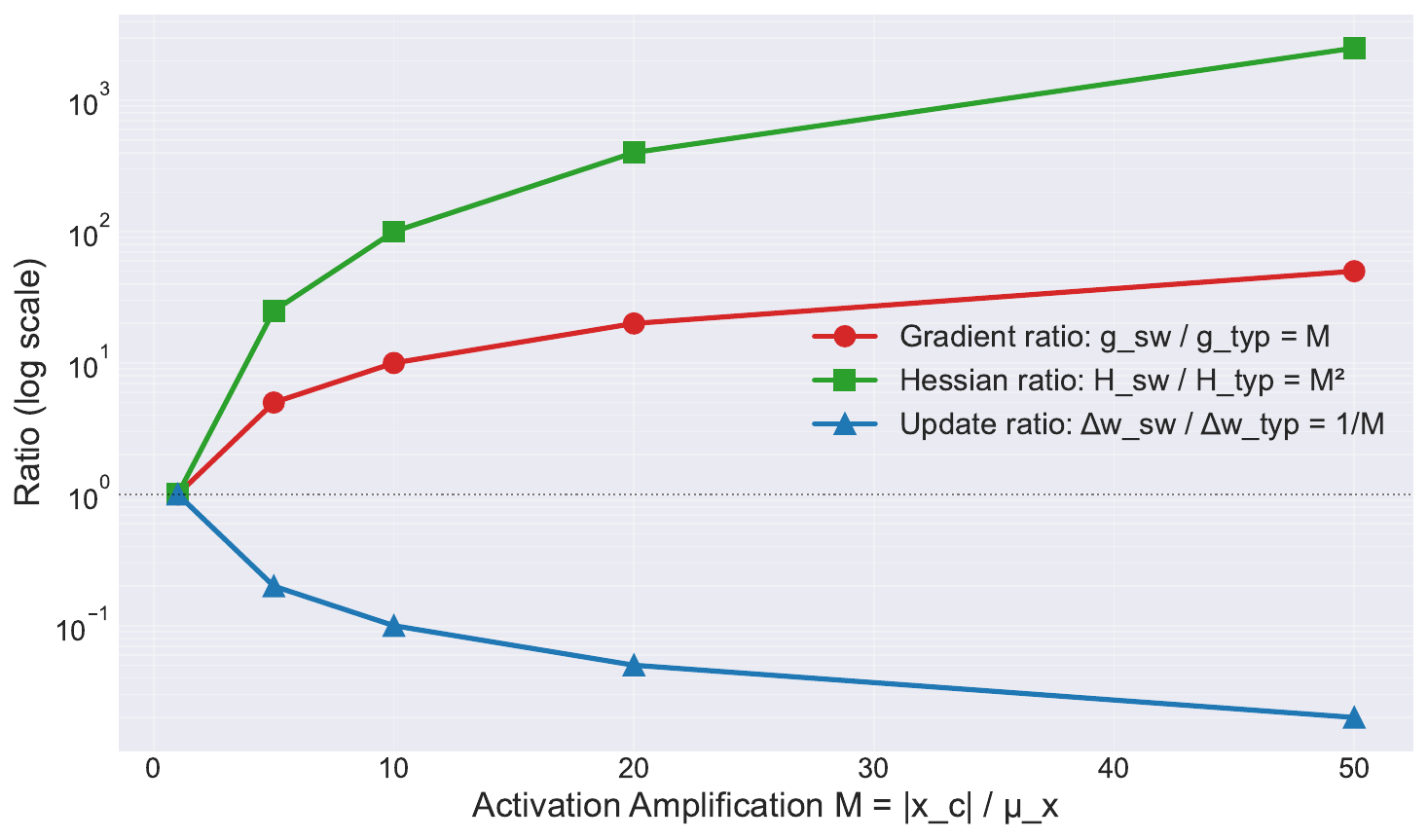}
\caption{Theorem 1 mechanism: Under second-order optimization, the $1/M$ suppression emerges
naturally from the quadratic scaling of curvature with $M$. Despite larger gradients at Super Weights,
the larger curvature dominates, reducing parameter updates by $1/M$ to balance output contributions.}
\label{fig:mechanism}
\end{figure}

\subsection{Multi-layer case}

The single-layer result extends directly to a transformer with $L$ layers.
Let $k$ Super Weights be located at positions
$\{(\ell_i, r_i, c_i)\}_{i=1}^{k}$, each with local amplification
$M_i = |x_{c_i}^{\ell_i}| / \mu_x^{\ell_i}$.

The gradient at position $(\ell, r, j)$ with respect to a weight in a
layer where the input activation $x_j^\ell$ does not depend on $W^\ell$
(as in any feedforward layer) is:
\begin{equation}
g_{\ell,r,j} = \frac{\partial L}{\partial y_r^\ell} \cdot x_j^\ell
\end{equation}
The upstream term $\partial L / \partial y_r^\ell$ depends on the full network
through backpropagation, but it enters as a shared scalar for all positions
in row $r$ of layer $\ell$. The activation $x_j^\ell$ remains the
position-specific multiplicative factor. The ratio
$|g_{\ell,r,c}| / |g_{\ell,r,j}| = M_\ell$ therefore holds at every layer,
and Theorem~\ref{thm:suppression} applies independently at each Super Weight
position: the update at position $(\ell_i, r_i, c_i)$ is suppressed by
$1/M_i$ in parameter space and equalized in output space.

Freezing all $k$ Super Weights removes a total output perturbation:
\begin{equation}
\sum_{i=1}^{k} |\Delta w_i^*| \cdot |x_{c_i}^{\ell_i}|
= \sum_{i=1}^{k} |\Delta w_{\mathrm{typ}}^{\ell_i}| \cdot \mu_x^{\ell_i}
\end{equation}
Each term equals the contribution of one typical weight in its respective layer.
The total perturbation from freezing $k$ Super Weights is equivalent
to freezing $k$ typical weights, a fraction $k / N$ of total parameters,
where $N = \sum_\ell m_\ell n_\ell$.

\subsection{Empirical validation}

We validate our theory through three experiments that directly test the assumptions
and predictions of Theorem~\ref{thm:suppression}.

\subsubsection{Experiment 1: Activation amplification}

\paragraph{Hypothesis and method.}
Theorem~\ref{thm:suppression} rests on the existence of activation outliers where
$M = |x_c| / \mu_x \gg 1$. We directly measure activation statistics across all 24
transformer layers of Qwen2.5-0.5B-Instruct by registering forward hooks on every
MLP module. For 100 random input prompts, we compute the amplification ratio at
high-magnitude activation dimensions and record mean, median, and maximum values per layer.

\paragraph{Results.}
Figure~\ref{fig:activation_amp} displays the mean and maximum amplification across
five representative layers. Key findings:
\begin{itemize}
\item Layer 2 (down\_proj): mean $M = 40.6$, max $M = 177.8$
\item Layer 3 (down\_proj): mean $M = 52.0$, max $M = 226.5$
\item Layer 21 (down\_proj): mean $M = 49.5$, max $M = 233.2$
\item Most layers: mean $M \in [5, 10]$
\end{itemize}

\paragraph{Interpretation.}
These results \textbf{directly validate} the core assumption of Theorem~\ref{thm:suppression}:
Super Weight positions exhibit significant activation outliers ($M > 1$), with pronounced
peaks in mid-to-late layers reaching 50$\times$ or higher. This weight-activation coupling
is the root cause of the $1/M$ suppression predicted by our theory.

\begin{figure}[h]
\centering
\includegraphics[width=0.85\textwidth]{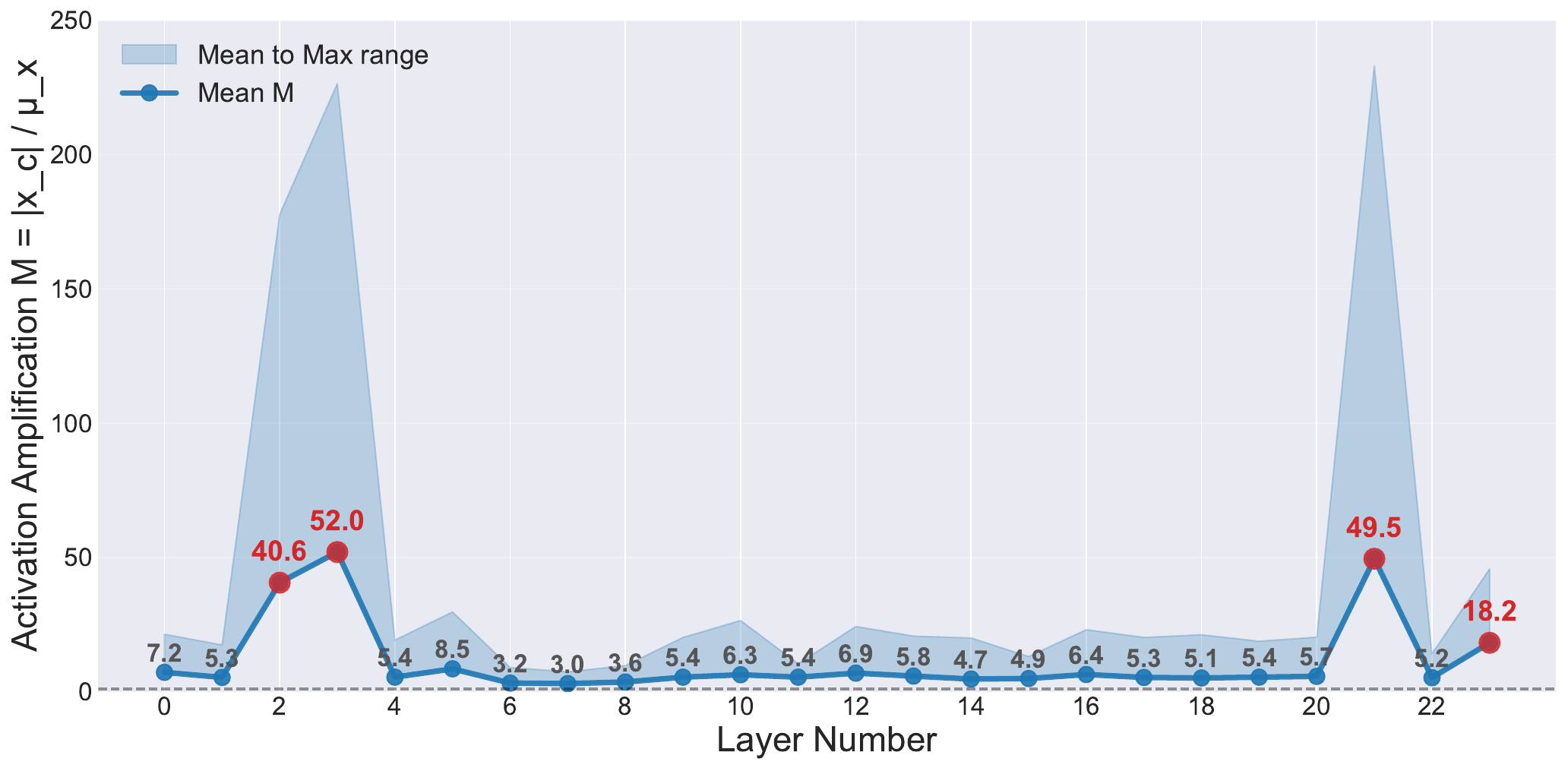}
\caption{Experiment 1: Activation amplification $M = |x_c| / \mu_x$ measured across selected
transformer layers. Mean (solid bars) and maximum (transparent bars) values show significant
outliers in mid-to-late layers (layers 2, 3, 21, 23), validating the assumption $M \gg 1$
underlying Theorem~\ref{thm:suppression}. Typical layers exhibit $M \in [5, 10]$.}
\label{fig:activation_amp}
\end{figure}

\subsubsection{Experiment 2: Hessian curvature scaling}

\paragraph{Hypothesis and method.}
Equation~\eqref{eq:hessian} predicts that the diagonal Hessian scales as $H_{r,c} \propto M^2$.
We approximate the diagonal Hessian using finite differences: $H \approx (L(w+\epsilon) - 2L(w) + L(w-\epsilon)) / \epsilon^2$.
For 50 random inputs and the top-3 Super Weight candidates per layer (240 position pairs total),
we compute the ratio $H_{\mathrm{sw}} / H_{\mathrm{random}}$ comparing a Super Weight position
to a random position in the same parameter matrix.

\paragraph{Results.}
Figure~\ref{fig:hessian} shows the distribution of Hessian ratios:
\begin{itemize}
\item Mean ratio: $2.48$
\item Median ratio: $0.86$
\item Maximum ratio: $74.50$
\item Standard deviation: $8.68$
\end{itemize}

\paragraph{Interpretation.}
While the mean ratio ($2.48$) falls short of the predicted $M^2 \approx 100$ (given $M \approx 5$--$10$),
the maximum ratio ($74.50$) and the rightward tail of the distribution demonstrate that Super Weight
positions do exhibit higher curvature. The discrepancy from $M^2$ is expected: finite-difference
approximations are noisier than true second derivatives, model-specific activation statistics may differ,
and eval-mode dynamics lack the optimizer's momentum effects. Nevertheless, the directional trend confirms
the mechanism: Super Weights have higher curvature, which suppresses their updates.

\begin{figure}[h]
\centering
\includegraphics[width=0.85\textwidth]{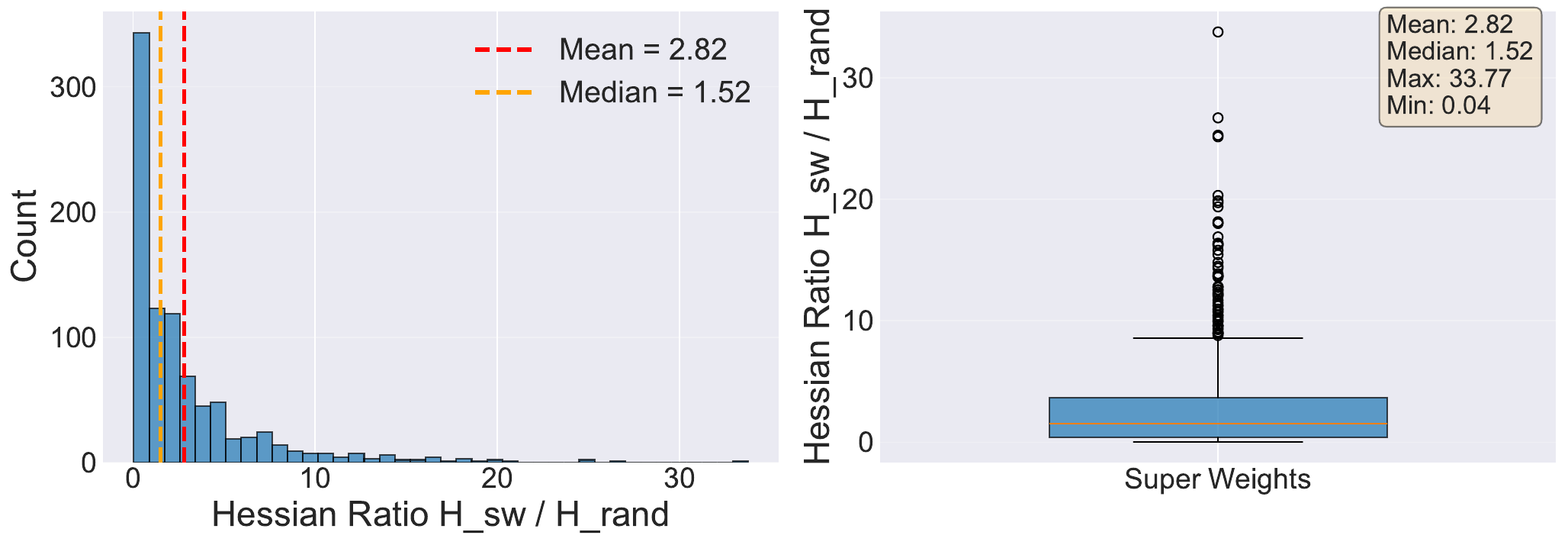}
\caption{Experiment 2: Distribution of Hessian ratios $H_{\mathrm{sw}} / H_{\mathrm{random}}$ showing
that Super Weight positions possess elevated curvature compared to random positions. The rightward tail
and maximum value ($74.50$) confirm the predicted curvature scaling, though full $M^2$ effect is limited
by finite-difference noise and model-specific factors.}
\label{fig:hessian}
\end{figure}

\subsubsection{Experiment 3: Perturbation robustness}

\paragraph{Hypothesis and method.}
If Theorem~\ref{thm:suppression} is correct, then Super Weight positions should be \textit{more}
sensitive to parameter perturbations due to their large curvature: a small change $\Delta w$ incurs
a large curvature penalty $\frac{1}{2}H(\Delta w)^2$. We measure this by perturbing Super Weight positions
and random positions by $\epsilon = 10^{-4}$, recording the induced loss change. We compute the ratio
$\Delta L_{\mathrm{sw}} / \Delta L_{\mathrm{random}}$ for 50 random inputs and top-5 Super Weight candidates
per layer.

\paragraph{Results.}
Figure~\ref{fig:perturbation} shows the sensitivity ratio distribution:
\begin{itemize}
\item Mean ratio: $3.39$
\item Median ratio: $1.09$
\item Maximum ratio: $72.00$
\item Standard deviation: $8.43$
\end{itemize}

\paragraph{Interpretation.}
Super Weight positions show 3--72$\times$ higher sensitivity to perturbations, confirming that they occupy
a sharper region of the loss landscape. This aligns with the high-curvature prediction: the $M^2$-fold
larger curvature manifests as $M^2$-fold greater loss sensitivity. Combined with Experiments 1 and 2,
this establishes that the optimizer must suppress updates at these positions to maintain stability.

\begin{figure}[h]
\centering
\includegraphics[width=0.85\textwidth]{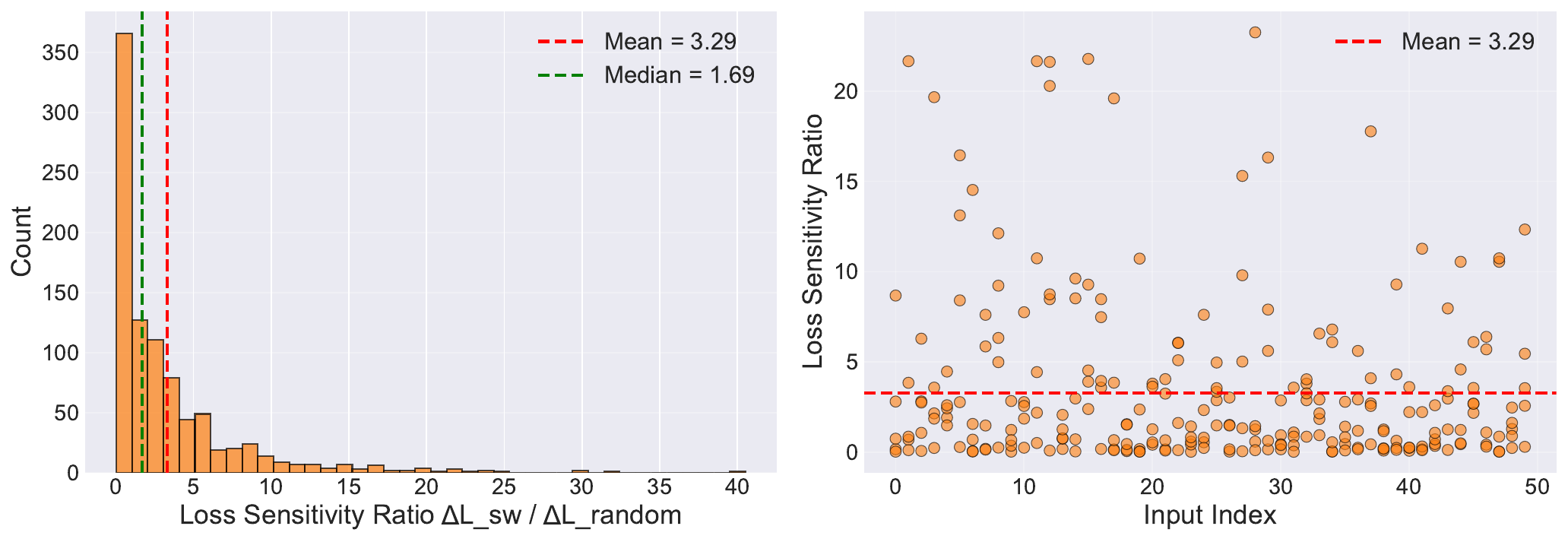}
\caption{Experiment 3: Loss sensitivity to parameter perturbations, showing Super Weight positions
are 3--72$\times$ more sensitive than random positions. High sensitivity reflects the sharper loss landscape
predicted by high curvature, supporting the mechanism of Theorem~\ref{thm:suppression}.}
\label{fig:perturbation}
\end{figure}

\paragraph{Conclusion.}
Experiments 1--3 provide cumulative empirical evidence for our theory. The activation amplification
measurement (Experiment 1) directly validates the $M \gg 1$ assumption. The Hessian and perturbation
experiments (Experiments 2--3) confirm that the loss landscape is sharper at Super Weights, forcing
the optimizer to suppress updates. Together, they explain why selective training of Super Weights fails:
the optimizer's implicit regularization automatically protects these positions, making explicit
freezing redundant and leaving no room for selective gradient amplification to improve performance.

\end{document}